\documentclass[final,3p,times]{elsarticle}
\UseRawInputEncoding 
\usepackage{amssymb}
\usepackage{amsthm}
\usepackage{amsmath}
\usepackage{mathrsfs}
\usepackage{stmaryrd}
\usepackage{graphicx,amsfonts,listings,xcolor,colortbl}
\usepackage{epstopdf}
\usepackage[linesnumbered,ruled,lined]{algorithm2e}
\usepackage{multirow}
\usepackage{booktabs}
\usepackage{blkarray}
\usepackage{subfig}
\usepackage{arydshln}
\usepackage{setspace}
\usepackage{float}

\newtheorem {lemma} {\bf Lemma}

\newtheorem {theorem} {\bf Theorem}[section]

\begin{document}

\begin{frontmatter}



\title{Adaptive Graph-based Generalized Regression Model for Unsupervised Feature Selection}%\tnoteref{NSFC}}
% Dynamic Maintenance of Rough Fuzzy Approximations with the Variation of Objects and Attributes
%\tnotetext[NSFC]{This work is supported by the National Science Foundation of China
%(No.~60873108).}


\author[SWUFE]{Yanyong Huang}
\ead{huangyy@swufe.edu.cn}
\author[SWUFE]{Zongxin Shen}
\ead{shenzx@smail.swufe.edu.cn}
\author[C]{Fuxu Cai }
\ead{caifuxu20@163.com}
\author[XIJIAO]{Tianrui Li \corref{cor1}}
\ead{trli@swjtu.edu.cn}
\author[SWUFE]{Fengmao Lv}
\ead{fengmaolv@126.com}

\cortext[cor1]{Corresponding author.}
\address[SWUFE]{School of  Statistics, Southwestern University of  Finance and Economics, Chengdu 611130, China }
\address[C]{Eye Institute, Putian College Affiliated Hospital, Putian 351100, China }
\address[XIJIAO]{Institute of Artificial Intelligence, School of Information Science and Technology, Southwest Jiaotong University, Chengdu 611756, China }

\begin{abstract}
Unsupervised feature selection is an important method to reduce dimensions of high dimensional data without labels, which is benefit to avoid ``curse of dimensionality'' and improve the performance of subsequent machine learning tasks, like clustering and retrieval. How to select the uncorrelated and discriminative features is the key problem of unsupervised feature selection. Many proposed methods select features with strong discriminant and high redundancy, or vice versa. However, they only satisfy one of these two criteria. Other existing methods choose the discriminative features with low redundancy by constructing the graph matrix on the original feature space. Since the original feature space usually contains redundancy and noise, it will degrade the performance of feature selection. In order to address these issues, we first present a novel generalized regression model imposed by an uncorrelated constraint and the $\ell_{2,1}$-norm regularization. It can simultaneously select the uncorrelated and discriminative features as well as reduce the variance of these data points belonging to the same neighborhood, which is help for the clustering task. Furthermore, the local intrinsic structure of data is constructed on the reduced dimensional space by learning the similarity-induced graph adaptively. Then the learnings of the graph structure and the indicator matrix based on the spectral analysis are integrated into the generalized regression model. Finally, we develop an alternative iterative optimization algorithm to solve the objective function.  A series of experiments are carried out on nine real-world data sets to demonstrate the effectiveness of the proposed method in comparison with other competing approaches.
\end{abstract}


\end{frontmatter}


%%
% \linenumbers

\section{Introduction}
Due to the rapid development of information and Internet technologies, the dimension of data  is increased dramatically in real applications, which will result in the curse of dimensionality and affect the performance of clustering, classification and so on. Feature selection~\cite{M. Ben-Bassat1982,L. Wolf 2005}, as an important dimensionality reduction method, can choose a part of useful features from the original high-dimensional feature space and get rid of  redundant and noisy features, which has been widely applied in image classification~\cite{3}, text clustering~\cite{4}, gene selection~\cite{5}, etc. In terms of the availability of the class label information, feature selection approaches are classified into supervised~\cite{YangYMa Z,ZhangRNieF}, semi-supervised~\cite{XiaoJCaoH,XuZKingI}, and unsupervised methods~\cite{DyJG,Wang S}. In some real applications, it is usually hard to obtain the labels of all samples. And manually labeling data is laborious and time consuming. Hence, unsupervised feature selection is a more challenging problem since the lack of label information in reality.

In recent years, there have been many research studies on the unsupervised feature selection. These methods can be divided into three categories: filter-based approaches~\cite{Battiti R,LapScore}, wrapper-based approaches~\cite{Mafarja M,Kohavi R} and embedded-based approaches~\cite{Zhu P,TangCLiuX}. Filter methods select the features by means of the corresponding scores computed by the evaluation measures, which describe the capacity of preserving the intrinsic structure of the unlabeled data. Wrapper methods choose some subsets of features to input a learning algorithm until  obtaining  the satisfactory performance. Embedded methods embody the unsupervised feature selection procedure  into the specific learning task, like clustering and dimensional reduction. It can simultaneously  learn the embedding subspace and features, as well as reduce the computational cost in comparison with the wrapper methods. There have been many research on unsupervised embedded feature selection methods~\cite{Wang S,Chenping Hou2014,ZhuX}. In embedded-based feature selection procedure, selecting discriminative features can distinguish different instances and choosing uncorrelated features can reduce the redundancy of data. In addition, maintaining  the local structure of data is benefit to improve the performance of subsequent tasks including clustering, classification, etc. In supervised scenario, available labels of instances can provide the discriminative information and serve to local structure learning directly. Since the lack of label information, the basic problem of unsupervised feature selection is how to select the features with discrimination and uncorrelation and preserve the local manifold geometry structure of data accurately. Hou et al. incorporated the learning of  low dimensional embedding structure to  the spectral regression model for the selection of features~\cite{Chenping Hou2014}. In order to alleviate the adverse effect of the noisy features and outliers, Shi et al. proposed a robust feature selection method named RSFS by learning the cluster structure based on the robust graph embedding and robust spectral regression model simultaneously~\cite{Lei Shi2014}. Based on non-negative spectral analysis, Li et al. proposed a joint framework for both the learning of discriminative feature matrix  and spectral clustering~\cite{NDFS}. Yang et al. chose the discriminative features by minimizing the total scatter matrix and maximizing the between class scatter matrix, as well as using $\ell_{2,1}$-norm for the regularization of the feature selection matrix~\cite{UDFS}. The abovementioned methods only concentrate on the selection of discriminative features which are the main role in the clustering or classification problems. However, these methods did not consider the high correlation of selected features, which has an adverse effect on clustering/classification results. In order to select the features with low redundancy, Zheng et al. presented a spectral feature selection method by using the sparse multi-output regression model with the constraint of $\ell_{2,1}$-norm~\cite{Z. Zhao2010}. Cai et al. investigated multiple local geometric structures of data by means of the spectral clustering method and employed LASSO regression model and MCFS score to choose the uncorrelated  features~\cite{D. Cai2010}. Gu et al. presented a joint framework by combining both feature selection and feature transformation, which is instructive to choose the uncorrelated features~\cite{Q. Gu2011}. These studies focus on selecting features with low redundancy, but neglect to construct the discriminative features. Besides, Li et al. proposed a discriminative and uncorrelated feature selection method by employing the side information described in the form of must-link and cannot-link constraints~\cite{DUCFS}. However, the aforementioned methods, regardless of these methods emphasizing on selecting discriminative features or uncorrelated features, the learning of local data geometric structure based on graph is conducted  on the original feature space. If the graph has been obtained from the original high-dimensional data, it will keep unchanged in the following procedure of  feature selection. Then the performance of feature selection methods heavily depends on the preconstructed graph.~\cite{Luo M,Du L}. Hence, it is important to develop the feature selection method to choose the discriminative and uncorrelated features with the adaptive construction of graph. There are very few studies focus on this issue. Li et al. presented a generalized uncorrelated regression model combining with adaptive graph construction to select the discriminative and uncorrelated features~\cite{X. Li2019}. Nevertheless, the graph-based local structure is constructed on the original feature space in~\cite{X. Li2019}. Due to the original feature space usually contains redundant and noisy data points, it will result in degraded performance of feature selection.

In order to deal with above-mentioned issues, this paper presents a novel unsupervised feature selection method by integrating an adaptive local structure learning into a generalized regression model imposed by an uncorrelated constraint for the selection of discriminative and uncorrelated features. Firstly, we propose a generalized uncorrelated regression model by forcing $\ell_{2,1}$-norm regularization and a novel uncorrelated constraint, where the orthogonal constraint is added to an extended scatter matrix. By this means, we can obtain the uncorrelated and discriminative features and avoid the singularity case while using the common scatter matrix. Besides, the variance of these data points belonging to the same neighborhood can be reduced. It is help for the clustering task, which is demonstrated in the experimental study. Moreover, the similarity-induced graph matrix and the indicator matrix based on the spectral analysis technique are integrated into the generalized regression model to adaptively learn the graph-based local structure constructed on the reduced dimensional space. Then, we develop an alternative iterative optimization algorithm to solve the objective function. Finally, comparative experiments are carried out on nine real-world data sets to demonstrate the effectiveness of the proposed method.

To sum up, the main contribution of this paper is described as follows:
(\romannumeral1) A generalized regression model coupled with a novel uncorrelated constraint is presented to select the discriminative and uncorrelated features. It can prevent the trivial solution of the model and improve the clustering performance in comparison with the models imposed by the common and previous extended uncorrelated constraints, respectively.
(\romannumeral2) The local geometric structure of data by the construction of similarity-induced graph matrix is combined into the generalized regression model for the feature selection and adaptive learning graph structure simultaneously.
(\romannumeral3) An efficient alternative optimization algorithm is developed to solve the proposed model and  the convergence and time complexity of the corresponding algorithm are analyzed. Experimental results on real datasets show the performance of the proposed method is better than that of other baseline unsupervised feature selection methods.

The rest of the paper is organized as follows. In Section 2, the notations used throughout the paper are introduced. In Section 3, we present the proposed unsupervised feature selection method based on the generalized regression model with adaptive graph. Section 4 introduces the detailed  alternative optimization algorithm. The convergence of the proposed algorithm is proved  and the corresponding time complexity is analyzed in Section 5. Section 6 shows the experimental results on benchmark datasets, and Section 7 concludes the paper.

\section{Notations}
In this section, we introduce the notations used throughout the paper. Let the boldface uppercase letter, boldface lowercase letter and plain italic denote the matrix, vector and scalar, respectively. Given a matrix $\mathbf{W} \in \mathbb{R}^{r \times t}$, $w_{i j}$ denotes the $(i, j)$-th entry of $\mathbf{W}$, $\mathbf{w}^{j}$ stands for the $j$-th row of matrix $\mathbf{W}$, and $\mathbf{w}_{i}$ indicates the corresponding $i$-th column of matrix $\mathbf{W}$. Let $\operatorname{Tr}(\mathbf{W})$, $\mathbf{W}^{T}$ and $\operatorname{Rank}(\mathbf{W})$ denote the trace of $\mathbf{W}$, the transpose of $\mathbf{W}$ and the rank of $\mathbf{W}$, respectively. The Frobenius norm of matrix  $\mathbf{W}$ is defined as $\|\mathbf{W}\|_{F}=\sqrt{\sum_{i j}\left|w_{i j}\right|^{2}}$. The $\ell_{2,1}$-norm is defined as $\|\mathbf{W}\|_{2,1}=\sum_{i=1}^{r} \sqrt{\sum_{j=1}^{t} w_{i j}^{2}}=\sum_{i=1}^{r}\left\|\mathbf{w}^{i}\right\|_{2}$, where $\left\|\mathbf{w}^{i}\right\|_{2}$ denotes the $\ell_{2}$-norm of vector $\mathbf{w}^{i}$. Let $\mathbf{H}=\mathbf{I}-(1 / n) \mathbf{1} \mathbf{1}^{T}$ denote the center matrix, where $\mathbf{I}$ is an identity matrix and $\mathbf{1}=[1,1, \ldots, 1]^{T} \in \mathbb{R}^{n}$. Given a data matrix $\mathbf{X}=\left\{\mathbf{x}_{1},\ldots,\mathbf{x}_{n}\right\} \in \mathbb{R}^{d \times n}$, where $d$ denotes the dimension of features and $n$ indicates the number of samples, then the data matrix $\mathbf{X}$ can be centralized by $\mathbf{H}\mathbf{X}$.

\section{Adaptive graph-based generalized regression model for unsupervised feature selection}
In this section, we present a generalized regression model equipping with a novel uncorrelated constraint to select the discriminative and uncorrelated features in the context of unsupervised learning. Furthermore, to learn the local geometric structure of data, an adaptive graph construction method and a graph regularization term based on the spectral analysis are integrated into the generalized regression model.

\subsection{Generalized regression model}
Given a data matrix  $\mathbf{X}=\left\{\mathbf{x}_{1}, \ldots, \mathbf{x}_{n}\right\} \in \mathbb{R}^{d \times n}$ and the corresponding label matrix $\mathbf{Y}=\left[y_{1},\ldots,y_{n}\right]^{T} \in \mathbb{R}^{n \times c}$, the traditional regression model is formulated as follows:
\begin{equation}\label{regression1}
\begin{aligned}
&\min _{\mathbf{W},\mathbf{b}}\left\|\mathbf{X}^{T} \mathbf{W}+\mathbf{1} \mathbf{b}^{T}-\mathbf{Y}\right\|_F^{2}+\lambda\|\mathbf{W}\|_\eta,\\
\end{aligned}
\end{equation}
where $\mathbf{W}=\left[w_{1}, w_{2}, \dots, w_{d}\right] \in \mathbb{R}^{d \times c}$ is the projection matrix, $\mathbf{b}\in \mathbb{R}^{c \times 1}$ is the bias, $\|\mathbf{W}\|_\eta$ indicates the regularization item and $\lambda$ denotes the corresponding regularization parameter. When $\eta$ is set to $F$-norm, this model is the classical ridge regression model~\cite{ridge}. While $\eta$ is set to $\ell_{1}$-norm, it becomes the Lasso regression model~\cite{lasso}. These models have been widely applied in machine learning tasks, including but not limited to classification, clustering and feature selection~\cite{Chenping Hou2014,regforclass,regforcluster}. However, these models cannot be applied directly in context of unsupervised learning where the label matrix $\mathbf{Y}$ is unknown. To obtain the label $\mathbf{Y}$ in the unsupervised scenario, it can be viewed as a variable to be optimized in model~(\ref{regression1}). But, it will result in a potential trial solution while $\mathbf{W}=\textbf{0}$, $\mathbf{b}=[1,0,\cdots,0]^{T}$ and $\mathbf{Y}=[1,0,\cdots,0]$. In order to deal with this issue, previous studies utilized the embedding indicator matrix to replace the unknown $\mathbf{Y}$ and then learned these parameters according to the supervised regression method~\cite{X. Li2019,ZhangH}. It can be described as follows.
\begin{equation}\label{3}
\begin{aligned}
&\min _{\mathbf{W}, \mathbf{F},\mathbf{b}}\left\|\mathbf{X}^{T} \mathbf{W}+\mathbf{1} \mathbf{b}^{T}-\mathbf{F}\right\|_{F}^{2}+\lambda\|\mathbf{W}\|_{2,1}\\
&{s.t.} \quad \operatorname{Rank}(\mathbf{W})=c, \mathbf{F}^{T}\mathbf{F}=\mathbf{I},
\end{aligned}
\end{equation}
where $\mathbf{F}=\left[f^{1}, f^{2}, \dots, f^{n}\right]^T \in \mathbb{R}^{n \times c}$ is the indicator matrix embedded in the data space and $\mathbf{W}$ is the subspace. The column full rank constraint is imposed on $\mathbf{W}$ to avoid the trivial solution and using $\ell_{2,1}$-norm of $\mathbf{W}$ is to ensure the row-sparsity for feature selection. Then $\left\|\mathbf{w}^{i}\right\|_{2} (i=\{1,2,\cdots,d\})$ describes the importance of the $i$-th feature, which can be employed to select the top $t (1\leq t\leq d)$ features. Besides, the orthogonal constraint of  $\mathbf{F}$ can avoid the trivial solution and remove the impact of scale effects. In this model, the $\mathbf{F}$ describes the cluster structure of data which can provide the discriminant information. The first term in model~(\ref{3}) tries to learn a subspace which makes the projected samples approach to the $\mathbf{F}$ as much as possible. Hence, this model can learn the discriminant features and the local structure of data simultaneously. As mentioned above, reducing feature redundancy is another issue of great concern to unsupervised feature selection. The traditional  methods imposed the uncorrelated constraint $\mathbf{W}^{T} \mathbf{R}_{t} \mathbf{W}=\mathbf{I}$, where $\mathbf{R}_{t}=\mathbf{X} \mathbf{H} \mathbf{X}^{T}$ denotes the total scatter matrix, to this model for the selection of uncorrelated features. However, the traditional uncorrelated constraint exists the following problems. It lacks of flexibility and selects the features with low redundancy will result in some discriminative features being lost, which has been verified in \cite{DUCFS,ZhangH,X. Li2019}. In addition, when the number of samples is smaller than the feature size, the matrix $\mathbf{R}_{t}$ is singular. In order to cope with these two issues, we present an extended uncorrelated constraint, which can avoid the abovementioned problems and show the effectiveness in the next experimental procedure. The novel constraint is defined as follows:

\begin{equation}
\mathbf{W}^{T} \mathbf{R}_{t}^{\prime} \mathbf{W}=\mathbf{I},
\end{equation}
where $\mathbf{R}_{t}^{\prime}=\mathbf{R}_{t}+\lambda_{w} \mathbf{D}_{w}+\alpha_{s}\mathbf{X}\mathbf{L}_{s}\mathbf{X}^T$, $\mathbf{L}_{s}=\mathbf{D}_{s}-\left(\mathbf{S}^{T}+\mathbf{S}\right)/2$ is the Laplacian matrix of the similarity matrix $\mathbf{S}\in \mathbb{R}^{n \times n}$, $\mathbf{D}_{s}$ is the diagonal matrix with $D_{s}(i, i)=\sum_{j=1}^{n} \frac{s_{i j}+s_{j i}}{2}$, $\mathbf{D}_{w}\in \mathbb{R}^{d \times d}$ is defined as a diagonal matrix with $\mathbf{D}_{w}(i, i)=\frac{1}{2 \sqrt{\left\|\mathbf{w}^{i}\right\|_{2}^{2}+\varepsilon}}$, and $\varepsilon$ is a small constant to keep the denominator from disappearing. Incorporating the proposed constraint into the regression model (\ref{3}), the generalized uncorrelated model is presented as follows:
\begin{equation}\label{6}
\begin{aligned}
&\min _{\mathbf{W}, \mathbf{F}, \mathbf{b}}\left\|\mathbf{X}^{T} \mathbf{W}+\mathbf{1 b}^{T}-\mathbf{F}\right\|_{F}^{2}+\lambda\|\mathbf{W}\|_{2,1}\\
& {s.t.} \quad\mathbf{W}^{T}\mathbf{R}_{t}^{\prime} \mathbf{W}=\mathbf{I},\mathbf{F}^{T} \mathbf{F}=\mathbf{I}.
\end{aligned}
\end{equation}

There are three terms of the proposed constraint.  The first term $\mathbf{W}^{T} \mathbf{R}_{t} \mathbf{W}$ makes the scatter matrix of the projected samples approach to an orthogonal matrix. Then the learned features are uncorrelated to a great extent. The second term $\lambda_{w} \mathbf{W}^{T} \mathbf{D}_{w} \mathbf{W}$ appending to the first term can avoid the situation where the traditional scatter matrix is singular. The third term $\alpha_{s} \mathbf{W}^{T} \mathbf{X}\mathbf{L}_{s}\mathbf{X}^T \mathbf{W}$ is benefit to reduce the variance between these samples in the same neighborhood under the graph structure. Especially, it can improve the performance of the clustering task, which will be shown in the following experimental results. Hence, the proposed generalized regression model can select the discriminative and uncorrelated features. Notice that when $\lambda_{w}=0$ and $\alpha_{s}=0$, the proposed model degenerates to the traditional uncorrelated constraint-based model. And when $\alpha_{s}=0$, it degenerates to the model presented by Li et al ~\cite{X. Li2019}.
\subsection{Adaptive graph learning for feature selection}
As discussed above, the presented regression model performs manifold learning in Euclidean space and applies the linear regression model to discover the low-dimensional structure. Since the manifold learning from the Euclidean space could not effectively explore the local geometrical structure, which is benefit to feature selection, clustering analysis, etc.~\cite{GuiJ,Stella}. In order to discover the local geometrical structure, previous studies have discussed some methods based on the spectral graph theory, where the local structure is characterized by the construction of nearest neighbor graph~\cite{LapScore,ZhuX,SOGFS}. The key point in constructing graph is the computation of  the corresponding similarity matrix, which is very important for unsupervised feature selection. However, in these methods, the similarity matrix is computed on the original feature space. Due to the original feature space usually contains redundancy and noise, it will degrade the performance of feature selection. Hence, in this section, we construct the similarity matrix in an adaptive way to alleviate the negative impact of redundant and noisy features.

Let $\mathbf{S}=(s_{i j})_{n \times n} \in R^{n \times n}$ denote the similarity matrix, where $s_{i j}$ indicates the similarity between the samples $\mathbf{x}_i$ and $\mathbf{x}_j$. The $\mathbf{S}$ can be determined by the $k$-nearest neighbor graph. Let $\mathbf{x}_j$ be the $k$-nearest neighbors of $\mathbf{x}_i$. Then the edge  is connected with regards to  $\mathbf{x}_i$ and $\mathbf{x}_j$. And the weight of the edge is determined by the similarity $s_{i j}$, which can be computed by different kernel functions, such as linear kernel, polynomial kernel and  Gaussian kernel~\cite{LapScore,Nonlinear,MJQian}. However, the computation of the similarity matrix in this method is sensitive to noise and outliers existing in the data. In the following, we present a novel construction of the similarity matrix with adaptive way, which is robust to noise and outliers. A natural assumption in manifold learning is that if two data points are close, then they are also close to each other in the embedding graph. Since there exists the redundancy and noisy features in the original data space, this assumption can be revised that if two data points are close in the dimension reduction space, then they are also close in the embedding graph. Based on this, the similarity matrix $\mathbf{S}$ can be constructed by the following problem:
\begin{equation}\label{4}
\begin{aligned}
&\min _{\mathbf{S}} \sum_{i, j=1}^{n}\left\|\mathbf{W}^{T}\mathbf{x}_{i}-\mathbf{W}^{T}\mathbf{x}_{j}\right\|_{2}^{2} s_{i j}\\
& {s.t.}\quad s_{i i}=0, s_{i j} \geq 0, \mathbf{1}^{T} \mathbf{s}_{i}=1.
\end{aligned}
\end{equation}

In the problem (\ref{4}), the sparse constraint $\mathbf{1}^{T} \mathbf{s}_{i}=1$ has been demonstrated that it can improve the robustness for the noise and outliers in~\cite{WangH}. However, the problem (\ref{4}) will have a trivial solution when there is only one data point in the dimension reduction space with the smallest distance to $\mathbf{x}_i$ having the value 1 and the other data points having the value 0. According to ~\cite{NieFWangX}, a prior is appended to model (\ref{4}), which is described as follows.

\begin{equation}\label{5}
\begin{aligned}
&\min _{\mathbf{S}} \sum_{i, j=1}^{n}\left\|\mathbf{W}^{T}\mathbf{x}_{i}-\mathbf{W}^{T}\mathbf{x}_{j}\right\|_{2}^{2} s_{i j}+\beta \sum_{i=1}^{n}\left\|\mathbf{s}_{i}\right\|_{2}^{2}\\
& {s.t.}\quad s_{i i}=0, s_{i j} \geq 0, \mathbf{1}^{T} \mathbf{s}_{i}=1.
\end{aligned}
\end{equation}
When only considering the second term in problem (\ref{5}), the prior can be interpreted as the similarity value associated with each data point to $\mathbf{x}_i$ equals to $\frac{1}{n}$.

Moreover, the other common assumption used in spectral graph analysis is that if two data points are close to each other in the intrinsic graph of data, then the corresponding embedding labels will are close to each other~\cite{LiZLiuJ,FengY}. In model (\ref{6}), the embedding labels are represented as the indicator matrix $\mathbf{F}$. Then we can add a graph regularization term to problem (\ref{5}), namely,

\begin{equation}\label{7}
\begin{aligned}
&\min _{\mathbf{S},\mathbf{F}} \sum_{i, j=1}^{n}\left\|\mathbf{W}^{T}\mathbf{x}_{i}-\mathbf{W}^{T}\mathbf{x}_{j}\right\|_{2}^{2} s_{i j}+\beta \sum_{i=1}^{n}\left\|\mathbf{s}_{i}\right\|_{2}^{2}+\alpha \operatorname{Tr}(\mathbf{F}^{T} \mathbf{L}_{s} \mathbf{F})\\
&{s.t.}\quad s_{i i}=0, s_{i j} \geq 0, \mathbf{1}^{T} \mathbf{s}_{i}=1,\mathbf{F}^T\mathbf{F}=\mathbf{I},
\end{aligned}
\end{equation}
where $\alpha$ is a regularization parameter. The model (\ref{7}) can keep the abovementioned assumptions and adaptively computes the similarity matrix along with the improvement of robustness.

Then, by incorporating the model (\ref{7}) into the generalized regression model (\ref{6}), we can obtain the unified model with adaptive graph regularization for unsupervised feature selection, which is described as follows:

\begin{equation}\label{8}
\begin{aligned}
&\min _{\mathbf{W}, \mathbf{F},\mathbf{S}, \mathbf{b}}\left\|\mathbf{X}^{T} \mathbf{W}+\mathbf{1 b}^{T}-\mathbf{F}\right\|_{F}^{2}+\lambda\|\mathbf{W}\|_{2,1}+\frac{1}{2}\alpha(\sum_{i, j=1}^{n}\left\|\mathbf{W}^{T}\mathbf{x}_{i}-\mathbf{W}^{T}\mathbf{x}_{j}\right\|_{2}^{2} s_{i j}+\beta \sum_{i=1}^{n}\left\|\mathbf{s}_{i}\right\|_{2}^{2}+\operatorname{Tr}(\mathbf{F}^{T} \mathbf{L}_{s} \mathbf{F}))\\
&{s.t.}\quad \mathbf{W}^{T}\mathbf{R}_{t}^{\prime}\mathbf{W}=\mathbf{I},\mathbf{F}^{T} \mathbf{F}=\mathbf{I},s_{i i}=0, s_{i j} \geq 0, \mathbf{1}^{T} \mathbf{s}_{i}=1.
\end{aligned}
\end{equation}

As we can see, the proposed unified model can select the uncorrelated  features under these constraints and adaptively learn the local geometrical structure, where the similarities between these data points are simultaneously preserved in the reduction dimensional space and the embedding graph structure. Moreover, it can keep the similar objects in the local graph structure with the similar labels learned from the generalized regression model.

\section{Optimization algorithm}
In this section, we employ an alternative method to optimize these variables in problem (\ref{8}), i.e., optimizing the objective function with regards to one variable while fixing the other variables and repeating the procedure until convergence.

Since there is no constraint on $\mathbf{b}$, it can be solved by the Karush-Kuhn-Tucker (KKT) conditions~\cite{W. Karush1939} where the first-order derivate of the Lagrangian function w.r.t. the objective function in (\ref{8}) equals to zero. Then the problem (\ref{8}) can be simplified. Concretely, let $\mathcal{L}(\mathbf{b})=\left\|\mathbf{X}^{T} \mathbf{W}+\mathbf{1} \mathbf{b}^{T}-\mathbf{F}\right\|_{F}^{2}+\mathcal{R}(\mathbf{W}, \mathbf{S}, \mathbf{F})$ denote the Lagrangian function of problem (\ref{8}) with respect to $\mathbf{b}$, where $\mathcal{R}(\mathbf{W}, \mathbf{S}, \mathbf{F})$ stands for the terms independent of $\mathbf{b}$ and dependent of  $\mathbf{W}$, $\mathbf{S}$ and $\mathbf{F}$. Taking the paritial derivate of $\mathcal{L}(\mathbf{b})$ and setting it to zero, we can obtain the optimal solution of $\mathbf{b}$ as follows:

\begin{equation}\label{9}
\frac{\partial \mathcal{L}(\mathbf{b})}{\partial \mathbf{b}}=\mathbf{0}\Rightarrow \mathbf{b}=\frac{1}{n}\left(\mathbf{F}^{\mathbf{T}}-\mathbf{W}^{\mathbf{T}} \mathbf{X}\right) \mathbf{1}.
\end{equation}

Substituting (\ref{9}) into (\ref{8}), then the problem (\ref{8}) is rewritten as follows:

\begin{equation}\label{10}
\begin{aligned}
&\min _{\mathbf{W}, \mathbf{F},\mathbf{S}}\left\|\mathbf{H}(\mathbf{X}^{T} \mathbf{W}-\mathbf{F})\right\|_{F}^{2}+\lambda\|\mathbf{W}\|_{2,1}+\frac{1}{2}\alpha(\sum_{i, j=1}^{n}\left\|\mathbf{W}^{T}\mathbf{x}_{i}-\mathbf{W}^{T}\mathbf{x}_{j}\right\|_{2}^{2} s_{i j}+\beta \sum_{i=1}^{n}\left\|\mathbf{s}_{i}\right\|_{2}^{2}+\operatorname{Tr}(\mathbf{F}^{T} \mathbf{L}_{s} \mathbf{F}))\\
&{s.t.}\quad \mathbf{W}^{T}\mathbf{R}_{t}^{'}\mathbf{W}=\mathbf{I},\mathbf{F}^{T} \mathbf{F}=\mathbf{I},s_{i i}=0, s_{i j} \geq 0, \mathbf{1}^{T} \mathbf{s}_{i}=1,
\end{aligned}
\end{equation}
where $\mathbf{H}=\mathbf{I}-(1 / n) \mathbf{1} \mathbf{1}^{T}$ is a symmetric centering matrix. In the following subsection, we introduce the alternating iterative algorithm to solve the variables $\mathbf{W}$, $\mathbf{F}$ and $\mathbf{S}$.

\subsection{Fix $\mathbf{F}$ and $\mathbf{S}$, update $\mathbf{W}$ }
When $\mathbf{F}$ and $\mathbf{S}$ are fixed, problem (\ref{10}) can be rewritten as:

\begin{equation}\label{11}
\begin{aligned}
&\min _{\mathbf{W}}\left\|\mathbf{H}(\mathbf{X}^{T} \mathbf{W}-\mathbf{F})\right\|_{F}^{2}+\lambda\|\mathbf{W}\|_{2,1}+\frac{1}{2}\alpha\sum_{i, j=1}^{n}\left\|\mathbf{W}^{T}\mathbf{x}_{i}-\mathbf{W}^{T}\mathbf{x}_{j}\right\|_{2}^{2} s_{i j}\\
&{s.t.}\quad \mathbf{W}^{T}\mathbf{R}_{t}^{\prime}\mathbf{W}=\mathbf{I}.
\end{aligned}
\end{equation}

\begin{lemma}\cite{X. Li2019}
$\lim _{\boldsymbol{w}^{i} \rightarrow \boldsymbol{w}_{*}^{i}} \frac{\partial\|\boldsymbol{W}\|_{2,1}}{\partial \boldsymbol{w}^{i}}=\frac{\boldsymbol{w}_{*}^{i}}{\left\|\boldsymbol{w}_{*}^{i}\right\|_{2}}=\lim _{\boldsymbol{w}^{i} \rightarrow \boldsymbol{w}_{*}^{i}, \varepsilon \rightarrow 0} \frac{\partial \operatorname{Tr}\left(\boldsymbol{W}^{T} \boldsymbol{D} \boldsymbol{W}\right)}{\partial \boldsymbol{w}^{i}} $, where $\mathbf{D}=\operatorname{diag}\left(\frac{1}{2 \sqrt{\left\|\mathbf{w}^{1}\right\|_{2}^{2}+\varepsilon}}, \frac{1}{2 \sqrt{\left\|\mathbf{w}^{2}\right\|_{2}^{2}+\varepsilon}}, \ldots, \frac{1}{2 \sqrt{\left\|\mathbf{w}^{d}\right\|_{2}^{2}+\varepsilon}}\right).$
\end{lemma}

Since $\frac{\partial \|\mathbf{W}\|_{2,1}}{\partial \mathbf{w}^{i}}=\frac{\mathbf{w}^{i}}{\left\|\mathbf{w}^{i}\right\|_{2}}$, it will result in the non-differentiable problem in (\ref{11}). In order to prevent this, we can transform problem (\ref{11}) into the following equivalent form by utilizing Lemma 1:

\begin{equation}\label{12}
\begin{aligned}
&\min _{\mathbf{W}}\left\|\mathbf{H}(\mathbf{X}^{T} \mathbf{W}-\mathbf{F})\right\|_{F}^{2}+\lambda\operatorname{Tr}(\mathbf{W}^{T} \mathbf{D}_{w}\mathbf{W})+\frac{1}{2}\alpha\sum_{i, j=1}^{n}\left\|\mathbf{W}^{T}\mathbf{x}_{i}-\mathbf{W}^{T}\mathbf{x}_{j}\right\|_{2}^{2} s_{i j}\\
&{s.t.}\quad \mathbf{W}^{T}\mathbf{R}_{t}^{'}\mathbf{W}=\mathbf{I}.
\end{aligned}
\end{equation}

By means of these two properties $\left\|\mathbf{A}\right\|_{F}^{2}=\operatorname{Tr}(\mathbf{A}^{T}\mathbf{A})$ and $\sum_{i, j=1}^{n}\left\|\mathbf{z}^{i}-\mathbf{z}^{j}\right\|_{2}^{2} s_{i j}=2\operatorname{Tr}(\mathbf{Z}^{T}\mathbf{L}_{s}\mathbf{Z})$, we have
\begin{equation}\nonumber
\begin{aligned}
&\min _{\mathbf{W}}\left\|\mathbf{H}(\mathbf{X}^{T} \mathbf{W}-\mathbf{F})\right\|_{F}^{2}+\lambda\operatorname{Tr}(\mathbf{W}^{T} \mathbf{D}_{w}\mathbf{W})+\frac{1}{2}\alpha\sum_{i, j=1}^{n}\left\|\mathbf{W}^{T}\mathbf{x}_{i}-\mathbf{W}^{T}\mathbf{x}_{j}\right\|_{2}^{2} s_{i j}\nonumber\\
\Leftrightarrow&\min _{\mathbf{W}}\operatorname{Tr(\mathbf{W}^{T}\mathbf{X}\mathbf{H}\mathbf{X}^{T}\mathbf{W}-2\mathbf{W}^{T}\mathbf{X}\mathbf{H}\mathbf{F}+\mathbf{F}^{T}\mathbf{H}\mathbf{F})}+\lambda\operatorname{Tr(\mathbf{W}^{T}\mathbf{D}_{w}\mathbf{W})}+\alpha\operatorname{Tr(\mathbf{W}^{T}\mathbf{X}\mathbf{L}_{s}\mathbf{X}^{T}\mathbf{W})}\nonumber\\
\Leftrightarrow&\min _{\mathbf{W}}\operatorname{Tr(\mathbf{W}^{T}(\mathbf{X}\mathbf{H}\mathbf{X}^{T}+\lambda\mathbf{D}_{w}+\alpha\mathbf{X}\mathbf{L}_{s}\mathbf{X}^{T})\mathbf{W})-2\operatorname{Tr(\mathbf{W}^{T}\mathbf{X}\mathbf{H}\mathbf{F})}}\nonumber\\
\Leftrightarrow&\min _{\mathbf{W}}\operatorname{Tr(\mathbf{W}^{T}\mathbf{R}_{t}^{\prime} {W})}-2\operatorname{Tr(\mathbf{W}^{T}\mathbf{X}\mathbf{H}\mathbf{F})}\nonumber\\
\Leftrightarrow & \min _{\mathbf{W}}-2\operatorname{Tr(\mathbf{W}^{T}\mathbf{X}\mathbf{H}\mathbf{F})}\nonumber\\
\Leftrightarrow&\max _{\mathbf{W}}\operatorname{Tr(\mathbf{W}^{T}\mathbf{X}\mathbf{H}\mathbf{F})}\label{13}
\end{aligned}
\end{equation}

Then problem~(\ref{12}) is equivalent to solve the following problem.
\begin{align}
&\max _{\mathbf{W}}\operatorname{Tr(\mathbf{W}^{T}\mathbf{X}\mathbf{H}\mathbf{F})}\label{13}\\
&\qquad{s.t.}\quad \mathbf{W}^{T}\mathbf{R}_{t}^{\prime}\mathbf{W}=\mathbf{I}\nonumber.
\end{align}
In order to solve problem (\ref{13}), Huang et al. has shown the optimal solution $\mathbf{A}=\mathbf{U}\mathbf{V}^{T}$ with regards to the problem $\max _{\mathbf{A}\mathbf{A}^{T}=\mathbf{I}}\operatorname{Tr(\mathbf{A}^{T}\mathbf{B})}$, where $\mathbf{U}$ and $\mathbf{V}$ are the left and right singular matrices of SVD decomposition on $\mathbf{B}$~\cite{Jin Huang2014}. Hence, let $\mathbf{A}=(\mathbf{R}_{t}^{\prime})^{\frac{1}{2}}\mathbf{W}$ and $\mathbf{B}=(\mathbf{R}_{t}^{\prime})^{-\frac{1}{2}}\mathbf{X}\mathbf{H}\mathbf{F}$, then we have the optimal solution of problem (\ref{12})  $\mathbf{W}=(\mathbf{R}_{t}^{\prime})^{-\frac{1}{2}}\mathbf{A}$. Since $\mathbf{D}_{w}$ also depends on $\mathbf{W}$, an iterative algorithm is developed as shown in Algorithm 1 to solve $\mathbf{W}$ in problem (\ref{11}).

\begin{algorithm}
\caption{Algorithm to solve $\mathbf{W}$ in problem (\ref{11})}
\KwIn{ \begin{enumerate}
         \item The coefficients $\lambda$ and $\alpha$;
         \item The data matrix $\mathbf{X}$ and centering matrix $\mathbf{H}$;
         \item The indicator matrix $\mathbf{F}$ and Laplacian matrix $\mathbf{L}_{s}$.
        \end{enumerate}
}
\KwOut{The feature selection matrix $\mathbf{W}\in \mathbb{R}^{d \times c}$.}

$\mathbf{Initialize:}$ $\mathbf{D}_{w}=\mathbf{I}\in \mathbb{R}^{d \times d}$.

\Begin
{
\While{not convergent}{

\quad With current $\mathbf{D}_{w}$, compute $\mathbf{R}_{t}^{\prime}=\mathbf{X}\mathbf{H}\mathbf{X}^{T}+\lambda_{w} \mathbf{D}_{w}+\alpha_{s}\mathbf{X}\mathbf{L}_{s}\mathbf{X}^T$;\\

\quad Calculate $\mathbf{B}=(\mathbf{R}_{t}^{\prime})^{-\frac{1}{2}}\mathbf{X}\mathbf{H}\mathbf{F}$;\\

\quad Solve the compact SVD decomposition on  $\mathbf{B}$ such that  $\mathbf{B}=\mathbf{U}\sum \mathbf{V}^{T}$;\\

\quad Compute $\mathbf{A}=\mathbf{U}\mathbf{V}^{T}$;\\

\quad Update $\mathbf{W}\leftarrow(\mathbf{R}_{t}^{\prime})^{-\frac{1}{2}}\mathbf{A}$;\\

\quad Update $\mathbf{D}=\operatorname{diag}\left(\frac{1}{2 \sqrt{\| \mathbf{w}^{1} \|_{2}^{2}+\varepsilon}}, \frac{1}{2 \sqrt{\| \mathbf{w}^{2} \|_{2}^{2}+\varepsilon}}, \cdots, \frac{1}{2 \sqrt{\| \mathbf{w}^{d} \|_{2}^{2}+\varepsilon}}\right)$.\\
    }
}
\end{algorithm}

\subsection{Fix $\mathbf{W}$ and $\mathbf{S}$ ,update $\mathbf{F}$ }
When updating $\mathbf{F}$ with the fixed $\mathbf{W}$ and $\mathbf{S}$,  we need to solve the following problem:

\begin{equation}\label{14}
\begin{aligned}
&\min _{\mathbf{F}}\left\|\mathbf{H}(\mathbf{X}^{T} \mathbf{W}-\mathbf{F})\right\|_{F}^{2}+\frac{1}{2}\alpha\operatorname{Tr(\mathbf{F}^{T}\mathbf{L}_{s}\mathbf{F})}\\
&{s.t.}\quad \mathbf{F}^{T}\mathbf{F}=\mathbf{I}.
\end{aligned}
\end{equation}

By utilization of the property of matrix $\left\|\mathbf{A}\right\|_{F}^{2}=\operatorname{Tr}(\mathbf{A}^{T}\mathbf{A})$, we have
\begin{align}
&\min _{\mathbf{F}^{T}\mathbf{F}=\mathbf{I}}\left\|\mathbf{H}(\mathbf{X}^{T} \mathbf{W}-\mathbf{F})\right\|_{F}^{2}+\frac{1}{2}\alpha\operatorname{Tr(\mathbf{F}^{T}\mathbf{L}_{s}\mathbf{F})}\nonumber\\
\Leftrightarrow&\min _{\mathbf{F}^{T}\mathbf{F}=\mathbf{I}}\operatorname{Tr(\mathbf{W}^{T}\mathbf{X}\mathbf{H}\mathbf{X}^{T}\mathbf{W}-2\mathbf{W}^{T}\mathbf{X}\mathbf{H}\mathbf{F}+\mathbf{F}^{T}\mathbf{H}\mathbf{F})+\frac{1}{2}\alpha\operatorname{Tr(\mathbf{F}^{T}\mathbf{L}_{s}\mathbf{F})}}\nonumber\\
\Leftrightarrow&\min _{\mathbf{F}^{T}\mathbf{F}=\mathbf{I}}\operatorname{Tr(\mathbf{F}^{T}\mathbf{Q}\mathbf{F}-2\mathbf{F}^{T}\mathbf{C})}\label{15},
\end{align}
where $\mathbf{Q}=\frac{\alpha}{2}\mathbf{L}_{s}+\mathbf{H}$, $\mathbf{C}=\mathbf{H}\mathbf{X}^{T}\mathbf{W}$. Obviously, problem (\ref{15}) is the standard quadratic problem on the Stiefel manifold which can be solved by the generalized power iteration method proposed by Nie et al.~\cite{F. Nie2017gpd}. Then the detailed algorithm is developed to solve problem (\ref{14}) in Algorithm 2.

\begin{algorithm}
\caption{Algorithm to solve $\mathbf{F}$ in problem (\ref{14})}
\KwIn{The matrices $\mathbf{Q}$ and $\mathbf{C}$ defined in (\ref{15})}
\KwOut{The indicator matrix $\mathbf{F}\in \mathbb{R}^{n \times c}$}
$\mathbf{Initialize:}$

(\romannumeral1) A random orthogonal matrix $\mathbf{F}\in \mathbb{R}^{n \times c}$ such that $\mathbf{F}^{T}\mathbf{ F}=\mathbf{I}$;\\
(\romannumeral2) Given a $\nu$ via power method~\cite{F. Nie2017gpd} such that $\widetilde{\mathbf{Q}}=\nu\mathbf{I}-\mathbf{Q} \in \mathbb{R}^{n \times n}$ satisfying positive definite definiteness.\\

\Begin
{
\While{not convergent}{

\quad Update $\mathbf{E}\leftarrow2\widetilde{\mathbf{Q}}\mathbf{F}+2\mathbf{C}$;\\

\quad Update $\mathbf{F}$ by solving the problem
$\max _{\mathbf{F}\mathbf{F}^{T}=\mathbf{I}}\operatorname{Tr(\mathbf{F}^{T}\mathbf{E})}$ according to the proposed method~\cite{Jin Huang2014}.\\
    }
}
\end{algorithm}

\subsection{Fix $\mathbf{W}$ and $\mathbf{F}$, update $\mathbf{S}$ }
With the fixed $\mathbf{W}$ and $\mathbf{F}$, problem (\ref{10}) is transformed to solve
\begin{equation}\label{16}
\begin{aligned}
&\min _{\mathbf{S}}\sum_{i, j=1}^{n}\left\|\mathbf{W}^{T}\mathbf{x}_{i}-\mathbf{W}^{T}\mathbf{x}_{j}\right\|_{2}^{2} s_{i j}+\beta \sum_{i=1}^{n}\left\|\mathbf{s}_{i}\right\|_{2}^{2}+\operatorname{Tr}(\mathbf{F}^{T} \mathbf{L}_{s} \mathbf{F})\\
&{s.t.}\quad s_{i i}=0, s_{i j} \geq 0, \mathbf{1}^{T} \mathbf{s}_{i}=1.
\end{aligned}
\end{equation}

By employing $\sum_{i, j=1}^{n}\left\|\mathbf{z}^{i}-\mathbf{z}^{j}\right\|_{2}^{2} s_{i j}=2\operatorname{Tr}(\mathbf{Z}^{T}\mathbf{L}_{s}\mathbf{Z})$, problem (\ref{16}) is equivalent to:

\begin{equation}\label{17}
\begin{aligned}
&\min _{\mathbf{S}}\sum_{i, j=1}^{n}(\left\|\mathbf{W}^{T}\mathbf{x}_{i}-\mathbf{W}^{T}\mathbf{x}_{j}\right\|_{2}^{2} +\frac{1}{2}\left\|\mathbf{f}^{i}-\mathbf{f}^{j}\right\|_{2}^{2})s_{i j}+\beta \sum_{i=1}^{n}\left\|\mathbf{s}_{i}\right\|_{2}^{2}\\
&{s.t.}\quad s_{i i}=0, s_{i j} \geq 0, \mathbf{1}^{T} \mathbf{s}_{i}=1.
\end{aligned}
\end{equation}

It can be seen that the problem (\ref{17}) is independent for different $i$. Then we can decouple problem (\ref{17}) into the following subproblem for each $i$.

\begin{equation}\label{18}
\begin{aligned}
&\min _{\mathbf{s}_{i}}\sum_{j=1}^{n}(\left\|\mathbf{W}^{T}\mathbf{x}_{i}-\mathbf{W}^{T}\mathbf{x}_{j}\right\|_{2}^{2} +\frac{1}{2}\left\|\mathbf{f}^{i}-\mathbf{f}^{j}\right\|_{2}^{2})s_{i j}+\beta \left\|\mathbf{s}_{i}\right\|_{2}^{2}\\
&{s.t.}\quad s_{i i}=0, s_{i j} \geq 0, \mathbf{1}^{T} \mathbf{s}_{i}=1.
\end{aligned}
\end{equation}

Let $g_{ij}=\left\|\mathbf{W}^{T}\mathbf{x}_{i}-\mathbf{W}^{T}\mathbf{x}_{j}\right\|_{2}^{2} +\frac{1}{2}\left\|\mathbf{f}^{i}-\mathbf{f}^{j}\right\|_{2}^{2}$ and the corresponding $i$-th coloumn vector is denoted as $\mathbf{g}_{i}$. Then problem (\ref{18}) is rewritten as

\begin{equation}\label{19}
\begin{aligned}
&\min _{\mathbf{s}_{i}}\sum_{j=1}^{n}\mathbf{g}_{ij}s_{i j}+\beta \left\|\mathbf{s}_{i}\right\|_{2}^{2}\\
&{s.t.}\quad s_{i i}=0, s_{i j} \geq 0, \mathbf{1}^{T} \mathbf{s}_{i}=1.
\end{aligned}
\end{equation}

Actually, problem (\ref{19}) is equivalent to solve the following problem, which differs by a constant term while $\mathbf{W}$ and $\mathbf{F}$ are fixed.

\begin{equation}\label{20}
\begin{aligned}
&\min _{\mathbf{s}_{i}}\frac{1}{2}\left\|\mathbf{s}_{i}+\frac{\mathbf{g}_{i}}{2 \beta}\right\|_{2}^{2} \\
&{s.t.}\quad s_{i i}=0, s_{i j} \geq 0, \mathbf{1}^{T} \mathbf{s}_{i}=1.
\end{aligned}
\end{equation}

The data points are more similar to their neighbours than to other non-neighbours in practice. Hence, we prefer the model in (\ref{20}) learning $\mathbf{s}_i$ with $k$ nonzero values, where $k$ is the number of neighbours. In what follows, we give the detailed solution steps.

Let the scalar $\psi$ and the vector $ \mathbf{\varphi} \geq 0$ be the lagrangian multipliers for the constraints $\mathbf{1}^{T} \mathbf{s}_{i}=1$ and $s_{i j} \geq 0$ in problem (\ref{20}), respectively. The Lagrangian function of problem (\ref{20}) is

\begin{equation}
\mathcal{L}(\mathbf{s}_{i}, \psi, \varphi)=\frac{1}{2}\left\|\mathbf{s}_{i}+\frac{\mathbf{g}_{i}}{2 \beta}\right\|_{2}^{2}-\psi(\mathbf{1}^{T}\mathbf{s}_{i}-1)-\mathbf{\varphi}^{T}\mathbf{s}_{i}.
\end{equation}

Taking the partial derivative of $\mathcal{L}$ w.r.t. $\mathbf{s}_{i}$ and setting it to zero, then

\begin{equation}\label{ds}
\mathbf{s}_{i}+\frac{\mathbf{g}_{i}}{2\beta}-\psi\mathbf{1}-\mathbf{\varphi}=0.
\end{equation}

For the $j$-th entry of $\mathbf{s}_{i}$ in Eq.(\ref{ds}), we obtain the following equations for $s_{i j}$:

\begin{equation}\label{ds1}
{s}_{i j}+\frac{g_{i j}}{2\beta}-\psi-\varphi_{j}=0.
\end{equation}

By the utilization of KKT condition $s_{i j}\varphi_{j}=0$, we have the solution of $s_{i j}$ denoted as $\hat{s}_{i j}$:
\begin{equation}\label{23}
\hat{s}_{i j}=(-\frac{g_{ij}}{2\beta}+\psi)_{+}.
\end{equation}

Assuming  $g_{i1},\ldots,g_{in}$ are sorted in an ascending order. Since the equality constraints $\mathbf{1}^{T} \mathbf{s}_{i}=1$, we can get

\begin{equation}\label{21}
\psi=\frac{1}{k}+\frac{1}{2k\beta}\sum_{j=1}^{k}g_{ij}.
\end{equation}

In order to obtain the $\mathbf{s}_{i}$ with the constraint of  $k$ non-zero entries, we prefer $\hat{s}_{ik}>0$ and $\hat{s}_{i,k+1}=0$. Then, we have

\begin{equation}\label{222}
\frac{-g_{ik}}{2\beta}+\psi>0 \quad and \quad\frac{-g_{i,k+1}}{2\beta}+\psi\leq0.
\end{equation}

According (\ref{21}) and (\ref{222}), we can get
\begin{equation}
\left\{\begin{array}{l}
\beta>\frac{kg_{ik}-\sum_{j=1}^{k}g_{ij}}{2}; \\
\beta \leq \frac{kg_{i,k+1}-\sum_{j=1}^{k}g_{ij}}{2}.
\end{array}\right.
\end{equation}

The optimal solution $\hat{\mathbf{s}}_{i j}$ having $k$ nonzero entries can obtain when $\beta$ is set as
\begin{equation}\label{22}
\beta=\frac{kg_{i,k+1}-\sum_{j}^{k}g_{ij}}{2}.
\end{equation}

Then, we can obtain the final solution of $\mathbf{s}_{i j}$ shown as follows according to Eqs. (\ref{23}), (\ref{21}) and (\ref{22}).

\begin{equation}\label{24}
s_{i j}=\left\{\begin{array}{cl}
\frac{g_{i, k+1}-g_{i j}}{k g_{i, k+1}-\sum_{h=1}^{k} g_{i h}} & j \leq k; \\
0 & j>k.
\end{array}\right.
\end{equation}

By optimizing each variable alternatively, the whole algorithm (AGUFS) to solve problem (\ref{8}) is summarized in Algorithm 3.

\begin{algorithm}
\caption{AGUFS to solve problem (\ref{8})}
\KwIn{ \begin{enumerate}
         \item The data matrix $\mathbf{X}\in \mathbb{R}^{d \times n}$ and the center matrix $\mathbf{H}$;
         \item The parameters $\alpha,\beta$ and $\lambda$;
         \item The number of neighbours $k$.
        \end{enumerate}
}
\KwOut{Selecting the top ranking $t$ features according to the descending order of $\left\|\mathbf{w}^{i}\right\|_{2}(i=1,2, \cdots, d)$.}

$\mathbf{Initialize:}$

(\romannumeral1) A random orthogonal matrix $\mathbf{F}\in \mathbb{R}^{n \times c}$;\\
(\romannumeral2) A random matrix  $\mathbf{W}\in \mathbb{R}^{d \times c}$;\\
(\romannumeral3) The similarity matrix $\mathbf{S}\in \mathbb{R}^{n \times n}$ by means of Eq.(\ref{24})\\

\Begin
{
\While{not convergent}{
\quad Fix $\mathbf{F}$ and $\mathbf{S}$,update $\mathbf{W}$ by Algorithm 1;\\
\quad Fix $\mathbf{W}$ and $\mathbf{S}$,update $\mathbf{F}$ by Algorithm 2;\\
\quad Fix $\mathbf{W}$ and $\mathbf{F}$,update $\mathbf{S}$ by Eq.(\ref{24}).\\
    }
}
\end{algorithm}

\section{Discussions}
In this section, we give the theoretical analysis of the convergence and computational complexity of the proposed algorithm, respectively.
\subsection{Convergence analysis}
\subsubsection{Convergence analysis of Algorithm 1}
In order to prove the convergence of Algorithm 1, we firstly introduce the following lemma.
\begin{lemma}\cite{NieFHuangH}\label{Lemma2}
For any two nonzero vectors $\mathbf{a}, \mathbf{b} \in \mathbb{R}^{c}$, the following inequality holds:
\begin{equation}
\|\mathbf{a}\|_{2}-\frac{\|\mathbf{a}\|_{2}^{2}}{2\|\mathbf{b}\|_{2}} \leq\|\mathbf{b}\|_{2}-\frac{\|\mathbf{b}\|_{2}^{2}}{2\|\mathbf{b}\|_{2}}.
\end{equation}
\end{lemma}

Based on the Lemma~\ref{Lemma2}, we give the proof of the convergence of Algorithm 1 by the following theorem.

\begin{theorem}
For each iterative updating $\mathbf{W}$ to its optimal solution in problem (\ref{12}), Algorithm 1 will decrease problem (\ref{11}) until convergence.
\end{theorem}

\begin{proof}
Let $\mathcal{J}(\mathbf{W}_{(t)},\mathbf{D}{w}_{(t)})$ denote the objective value of problem (\ref{12}) in the $t$-th iteration. Since Algorithm 1 updates $\mathbf{W}$ and $\mathbf{D}$ with the optimal solution in problem (\ref{12}), it always holds
\begin{equation}\label{25}
\begin{aligned}
&\mathcal{J}(\mathbf{W}_{(t+1)},\mathbf{D}_{w(t)}) \leq \mathcal{J}(\mathbf{W}_{(t)},\mathbf{D}_{w(t)})\\
\Rightarrow&\left\|\mathbf{H}(\mathbf{X}^{T} \mathbf{W}_{(t+1)}-\mathbf{F})\right\|_{F}^{2}+\lambda\operatorname{Tr}(\mathbf{W}_{(t+1)}^{T} \mathbf{D}_{w(t)}\mathbf{W})_{(t+1)}+\frac{1}{2}\alpha\sum_{i, j=1}^{n}\left\|\mathbf{W}_{(t+1)}^{T}\mathbf{x}_{i}-\mathbf{W}_{(t+1)}^{T}\mathbf{x}_{j}\right\|_{2}^{2} s_{i j}\\
&\leq \left\|\mathbf{H}(\mathbf{X}^{T} \mathbf{W}_{(t)}-\mathbf{F})\right\|_{F}^{2}+\lambda\operatorname{Tr}(\mathbf{W}_{(t)}^{T} \mathbf{D}_{w(t)}\mathbf{W}_{(t)})+\frac{1}{2}\alpha\sum_{i, j=1}^{n}\left\|\mathbf{W}_{(t)}^{T}\mathbf{x}_{i}-\mathbf{W}_{(t)}^{T}\mathbf{x}_{j}\right\|_{2}^{2} s_{i j}.
\end{aligned}
\end{equation}

By means of Lemma 2, we have

\begin{align}
&\|\mathbf{w}_{(t+1)}^{i}\|_{2}-\frac{1}{2\|\mathbf{w}_{(t)}^{i}\|_{2}}\|\mathbf{w}_{(t+1)}^{i}\|_{2}^{2} \leq\|\mathbf{w}_{(t)}^{i}\|_{2}-\frac{1}{2\|\mathbf{w}_{(t)}^{i}\|_{2}}\|\mathbf{w}_{(t)}^{i}\|_{2}^{2}\nonumber\\
\Rightarrow&\sum_{i=1}^{d}\|\mathbf{w}_{(t+1)}^{i}\|_{2}-\sum_{i=1}^{d}\frac{1}{2\|\mathbf{w}_{(t)}^{i}\|_{2}}\|\mathbf{w}_{(t+1)}^{i}\|_{2}^{2} \leq\sum_{i=1}^{d}\|\mathbf{w}_{(t)}^{i}\|_{2}-\sum_{i=1}^{d}\frac{1}{2\|\mathbf{w}_{(t)}^{i}\|_{2}}\|\mathbf{w}_{(t)}^{i}\|_{2}^{2}\nonumber\\
\Rightarrow&\|\mathbf{W}_{(t+1)}\|_{2,1}-\operatorname{Tr}(\mathbf{W}_{(t+1)}^{T}\mathbf{D}_{w}{(t)}\mathbf{W}_{(t+1)}) \leq \|\mathbf{W}_{(t)}\|_{2,1}-\operatorname{Tr}(\mathbf{W}_{(t)}^{T}\mathbf{D}_{w}{(t)}\mathbf{W}_{(t)})\nonumber\\
\Rightarrow&\lambda\|\mathbf{W}_{(t+1)}\|_{2,1}-\lambda\operatorname{Tr}(\mathbf{W}_{(t+1)}^{T}\mathbf{D}_{w}{(t)}\mathbf{W}_{(t+1)}) \leq \lambda\|\mathbf{W}_{(t)}\|_{2,1}-\lambda\operatorname{Tr}(\mathbf{W}_{(t)}^{T}\mathbf{D}_{w}{(t)}\mathbf{W}_{(t)})(\lambda>0)\label{26}.
\end{align}

Adding Eqs. (\ref{25}) and (\ref{26}), we can get

\begin{equation}
\begin{aligned}
&\left\|\mathbf{H}(\mathbf{X}^{T} \mathbf{W}_{(t+1)}-\mathbf{F})\right\|_{F}^{2}+\lambda\|\mathbf{W}_{(t+1)}\|_{2,1}+\frac{1}{2}\alpha\sum_{i, j=1}^{n}\left\|\mathbf{W}^{T}_{(t+1)}\mathbf{x}_{i}-\mathbf{W}^{T}_{(t+1)}\mathbf{x}_{j}\right\|_{2}^{2} s_{i j}\\
&\leq \left\|\mathbf{H}(\mathbf{X}^{T} \mathbf{W}_{(t)}-\mathbf{F})\right\|_{F}^{2}+\lambda\|\mathbf{W}_{(t)}\|_{2,1}+\frac{1}{2}\alpha\sum_{i, j=1}^{n}\left\|\mathbf{W}^{T}_{(t)}\mathbf{x}_{i}-\mathbf{W}_{(t)}^{T}\mathbf{x}_{j}\right\|_{2}^{2} s_{i j}.\\
\end{aligned}
\end{equation}

Then, we can see that the objective function value of problem (\ref{11}) monotonically decreases by Algorithm 1 in each iteration.
\end{proof}

\subsubsection{Convergence analysis of Algorithm 2}
In a similar way, the convergence proof of Algorithm 2 is given as follows.
\begin{theorem}
Algorithm 2 decreases problem (\ref{14}) by iteratively updating $\mathbf{F}$ with its optimal solution to problem $\max _{\mathbf{F}\mathbf{F}^{T}=\mathbf{I}}\operatorname{Tr(\mathbf{F}^{T}\mathbf{R})}$ until convergence, where $\mathbf{R}=2\widetilde{\mathbf{Q}}\mathbf{F}+2\mathbf{C},\widetilde{\mathbf{Q}}=\nu\mathbf{I}-\mathbf{Q}$.
\end{theorem}
\begin{proof}
Let $\mathcal{Y}(\mathbf{F}_{(t)})$ denote the objective function value of problem (\ref{15}) in the $t$-th iteration. Because Algorithm 2 obtains the optimal solution of the problem $\max _{\mathbf{F}\mathbf{F}^{T}=\mathbf{I}}\operatorname{Tr(\mathbf{F}^{T}\mathbf{R})}$, for the ($t+1$)-th iteration, it is easy to see that

\begin{equation}
\begin{aligned}
&\operatorname{Tr}(\mathbf{F}_{(t+1)}^{T}\mathbf{R}_{(t)}) \geq \operatorname{Tr}(\mathbf{F}_{(t)}^{T}\mathbf{R}_{(t)})\\
\Rightarrow&\operatorname{Tr}(\mathbf{F}_{(t+1)}^{T}\widetilde{\mathbf{Q}}\mathbf{F}_{(t)}+\mathbf{F}_{(t+1)}^{T}\mathbf{C}) \geq \operatorname{Tr}(\mathbf{F}_{(t)}^{T}\widetilde{\mathbf{Q}}\mathbf{F}_{(t)}+\mathbf{F}_{(t)}^{T}\mathbf{C}).\\
\end{aligned}
\end{equation}

Since $\widetilde{\mathbf{Q}}$ is initialized as a real symmetric positive definite matrix in \cite{F. Nie2017gpd}, we have $\widetilde{\mathbf{Q}}=\mathbf{Z}^{T}\mathbf{Z}$ by Choleshy factorization~\cite{Cholesky factorization}. Then,

\begin{equation}\label{27}
\begin{aligned}
\operatorname{Tr}(\mathbf{F}_{(t+1)}^{T}\mathbf{Z}^{T}\mathbf{Z}\mathbf{F}_{(t)}+\mathbf{F}_{(t+1)}^{T}\mathbf{C}) \geq \operatorname{Tr}(\mathbf{F}_{(t)}^{T}\mathbf{Z}^{T}\mathbf{Z}\mathbf{F}_{(t)}+\mathbf{F}_{(t)}^{T}\mathbf{C}).\\
\end{aligned}
\end{equation}

Furthermore, since $\left\|\mathbf{Z}\mathbf{F}_{(t+1)}-\mathbf{Z}\mathbf{F}_{(t)}\right\|_{F}^{2} \geq 0$, then we have
\begin{equation}\label{28}
\operatorname{Tr}(\mathbf{F}_{(t+1)}^{T}\mathbf{Z}^{T}\mathbf{Z}\mathbf{F}_{(t+1)})+\operatorname{Tr}(\mathbf{F}_{(t)}^{T}\mathbf{Z}^{T}\mathbf{Z}\mathbf{F}_{(t)})-2\operatorname{Tr}(\mathbf{F}_{(t+1)}^{T}\mathbf{Z}^{T}\mathbf{Z}\mathbf{F}_{(t)})\geq 0.
\end{equation}

By calculating $2 \times(\ref{27}) + (\ref{28})$, we obtain

\begin{equation}
\begin{aligned}
&2\operatorname{Tr}(\mathbf{F}_{(t+1)}^{T}\mathbf{C})+\operatorname{Tr}(\mathbf{F}_{(t+1)}^{T}\widetilde{\mathbf{Q}}\mathbf{F}_{(t+1)})\geq 2\operatorname{Tr}(\mathbf{F}_{(t)}^{T}\mathbf{C})+\operatorname{Tr}(\mathbf{F}_{(t)}^{T}\widetilde{\mathbf{Q}}\mathbf{F}_{(t)})\\
\Rightarrow&2\operatorname{Tr}(\mathbf{F}_{(t+1)}^{T}\mathbf{C})-\operatorname{Tr}(\mathbf{F}_{(t+1)}^{T}\mathbf{Q}\mathbf{F}_{(t+1)})\geq 2\operatorname{Tr}(\mathbf{F}_{(t)}^{T}\mathbf{C})-\operatorname{Tr}(\mathbf{F}_{(t)}^{T}\mathbf{Q}\mathbf{F}_{(t)})\\
\Rightarrow&\mathcal{Y}(\mathbf{F}_{(t+1)})\geq\mathcal{Y}(\mathbf{F}_{(t)}).
\end{aligned}
\end{equation}

Hence, in each iteration, the objective function values of problem (\ref{15}) decreased by Algorithm 2. Since problem (\ref{14}) is equivalent to problem (\ref{15}), then the proof of Theorem 5.2 is completed.
\end{proof}

In addition, it is easy to prove that the objective function of problem (\ref{17}) is convex with regards to $\mathbf{S}$. Hence, the closed-form solution of $\mathbf{S}$ employing the augmented Lagrange multiplier method can converge  to the globally optimal solution. In summary, the proposed Algorithm AGUFS is converge.

\subsection{Complexity analysis}
In this section, we analyze the computational complexity of the proposed algorithm AGUFS which consists of three parts. Specifically, while updating the feature selection matrix $\mathbf{W}$ by optimizing problem (\ref{13}), the time complexity is $\mathcal{O}(d^{3})$ in the Step 7 of Algorithm AGUFS. While computing the indicator matrix $\mathbf{F}$ by solving problem (\ref{15}), the time complexity is $\mathcal{O}(n^{2}d)$ in the Step 8 of Algorithm AGUFS. While the update of the similarity matrix $\mathbf{S}$ by optimizing problem (\ref{17}), the corresponding time complexity is $\mathcal{O}(nkd)$ in the Step 9 of Algorithm AGUFS, where $k$ is the number of neighbours. To sum up, the time complexity of AGUFS is $\mathcal{O}(d^{3}t+n^{2}dt+nkdt)$, where $t$ is the number of iteration. Hence, the computational complexity is cubic with the number of features and square with the number of samples. The main complexity is due to eigen-decomposition procedure, which also exists in other embedding-based feature methods ~\cite{Chenping Hou2014,X. Li2019,SOGFS}. It will result in the badly scaled  for the large-scale data. We will discuss this topic in our future work.

\section{Experiments}
In this section, we conduct a series of experiments on nine benchmark data sets to demonstrate the effectiveness of the proposed algorithm in comparison with  other competing methods.
\subsection{Experimental schemes}
\subsubsection{Datasets}
Nine real-world datasets chosen from different fields are used to evaluate the performance of the proposed method. It includes four face image datasets COIL20\footnote{https://jundongl.github.io/scikit-feature}~\cite{COIL20}, ${\rm Umist}^{1}$~\cite{Umist}, ${\rm JAFFE}^{1}$~\cite{JAFFE} and ${\rm WarpPIE10P}^{1}$~\cite{WarpPIE10P}, two biological datasets ${\rm Lung}^{1}$~\cite{Lung} and ${\rm Lymphoma}^{1}$~\cite{ly}, two handwritten digit datasets ${\rm USPS}^{1}$~\cite{USPS} and MFEA\footnote{http://archive.ics.uci.edu/ml/index.php}~\cite{MFEA}, and one sound dataset ${\rm Isolet}^{1}$~\cite{Isolet}. The detailed statistics of these datasets are shown in Table 1.

\begin{table}[!htb]
\tabcolsep 0pt
\caption{A detail description of datasets}
\vspace*{-15pt}
    \begin{flushleft}
    \def\temptablewidth{\textwidth}
        {\rule{\temptablewidth}{1pt}}
        \begin{tabular*}{\temptablewidth}{@{\extracolsep{\fill}}lllll}
            No.&Datasets & Samples & Features & Classes  \\
            \hline
            1 & Lymphoma     &  96  & 4026 & 9 \\
            2 & Lung         &  203  & 3312 & 5 \\
            3 & warpPIE10P   & 210  & 2420 & 10 \\
            4 & JAFFE        & 213  & 676 & 10 \\
            5 & Umist        &  575  & 644 & 20  \\
            6 & COIL20       &  1440 & 1024 & 20\\
            7 & Isolet       &  1560  & 617 & 26 \\
            8 & MFEA         &  2000  & 216 & 10 \\
            9 & USPS         &  9298  & 256 & 10 \\
        \end{tabular*}
        {\rule{\temptablewidth}{1pt}}
        \end{flushleft}
\end{table}

\subsubsection{Baselines}
In order to validate the effectiveness of the proposed method, comparisons are made with seven baseline methods of unsupervised feature selection. These approaches are briefly introduced as follows:

URAFS~\cite{X. Li2019}: Uncorrelated Regression with Adaptive graph for unsupervised Feature Selection (URAFS) employs the generalized uncorrelated regression model to construct the adaptive graph for the selection of discriminative and uncorrelated features.

NDFS~\cite{NDFS}: Nonnegative Discriminative Feature Selection selects features by a joint framework integrating nonnegative spectral analysis and $\ell_{2,1}$ -norm regularized regression model.

JELSR~\cite{Chenping Hou2014}: Joint Embedding Learning and Sparse Regression is a feature selection framework by the combination of the embedded learning and the sparse regression model.

UDFS~\cite{UDFS}: Unsupervised Discriminate Feature Selection selects features via integrating the discriminative analysis and  $\ell_{2,1}$-norm minimization.

RSFS~\cite{Lei Shi2014}: Robust Spectral Feature Selection is a robust spectral learning framework for unsupervised feature selection, which combines the robust graph embedding and robust sparse spectral regression model.

MCFS~\cite{D. Cai2010}: Multi-Cluster Feature Selection (MCFS) which selects features by spectral analysis and the sparse regression model.

LapScore~\cite{LapScore}: Laplacian Score selects features according to the best capacity of preserving the locality structure.

\subsubsection{Evaluation metrics}
Two clustering metrics including clustering accuracy (ACC) and normalized mutual information (NMI) are employed to evaluate the performance of the abovementioned unsupervised feature selection methods.

Clustering  accuracy (ACC)~\cite{ACC}: Let $y_{i}$ and $\widetilde{y_{i}}$ denote the true label and clustering result of the $i$-th sample $\mathbf{x}_i (i=1,2,\cdots,n)$, respectively. Then ACC is defined as follows:
\begin{equation}
\mathrm{ACC}=\frac{1}{n} \sum_{i=1}^{n} \delta\left(y_{i}, \operatorname{map}\left(\widetilde{y_{i}}\right)\right),
\end{equation}
where $\delta(x,y)$ is the delta function satisfying $\delta(x,y)=1$ for $x=y$, otherwise $\delta(x,y)=0$, and $\operatorname{map}(\cdot)$ is the permutation mapping function which maps each cluster index to the best ground true label. It can be computed by using Kuhn-Munkres algorithm\cite{KM}.

Normalized mutual information (NMI)~\cite{NMI}: Let $\mathcal{C}$ be the set of ground truth labels and $\mathcal{C}^{\prime}$ denote the set of obtained clusters. Then NMI is formulated as follows:

\begin{equation}
\operatorname{NMI}\left(\mathcal{C}, \mathcal{C}^{\prime}\right)=\frac{\operatorname{MI}\left(\mathcal{C}, \mathcal{C}^{\prime}\right)}{\max \left(H(\mathcal{C}), H\left(\mathcal{C}^{\prime}\right)\right)},
\end{equation}
where $\operatorname{MI}(\mathcal{C},\mathcal{C}^{\prime})$ is the mutual information of $\mathcal{C}$ and $\mathcal{C}^{\prime}$, $H(\mathcal{C})$ and $H(\mathcal{C}^{\prime})$ are the information entropies of $\mathcal{C}$ and $\mathcal{C}^{\prime}$, respectively.

The larger the values of these two metrics are, the better performance is.

\subsubsection{Comparison settings}
In order to keep fair comparison with other baseline methods in the experiments, we use the same strategy to determine the optimal parameters for all methods. Namely, all parameters are tuned by searching the grid $\left\{10^{-3}, 10^{-2}, 10^{-1}\right.$ $\left.1,10,10^{2}, 10^{3}\right\}$. And the number of nearest neighbors is searched in $\left\{5, 10, 15\right\}$. For the selected features using the abovementioned unsupervised feature selection methods, we utilize K-Means clustering algorithm~\cite{kmeans} to evaluate the performance of all compared methods in terms of ACC and NMI. To avoid the occasionality, we repeatedly perform K-Means clustering 30 times and record the means and standard deviations for all methods. In addition, the value $k$ in the K-Means clustering algorithm is set to the true number of classes for each dataset.

\subsection{Comparisons over clustering performance}
In this section, we compare the performance of different unsupervised feature selection methods by the clustering results in terms of ACC and NMI. Tables 2 and 3 show the clustering results on top 60 features under nine benchmark datasets. The bold values indicate the best performance among eight feature selection approaches. According to the experimental results, the proposed method AGUFS is superior to other unsupervised feature selection approaches in terms of ACC and NMI. Especially, the ACC of our method are at least 6\% higher than that of the runner-up method on the data sets WarPIE10P, JAFFE and MFEA. And  the NMI of our method are around 5\% better than that of the runner-up method on the data sets WarPIE10P, Umist, Lung and MFEA.

Furthermore, we perform clustering tasks under selecting different number of features. Figs. 1 and 2 show the clustering results ACC and NMI with the increase of the number of features from 20 to 180, respectively. In Figs. 1 and 2, the $x$-coordinate pertains to the number of selected features and $y$-coordinate pertains to ACC and NMI, respectively. From these two figures, we can find that our method outperforms other compared methods at most time. In particular, the optimal ACC and NMI of the proposed method AGUFS are around 5\% larger than that of the second best method. Hence, these experimental results demonstrate the superiority of our method.

\begin{table*}[!htbp]
\tabcolsep 0pt
\caption{Means and standard deviation of ACC for different unsupervised feature selection methods}\small \label{KNN}
\vspace*{-5pt}
    \begin{flushleft}
    \def\temptablewidth{\textwidth}
        {\rule{\temptablewidth}{1pt}}
        \begin{tabular*}{\temptablewidth}{@{\extracolsep{\fill}}lllllllll}
            %\hline
            \multicolumn{9}{c}{Part A: Means}\\
            Data sets  & AGUFS  &  URAFS& NDFS  & JELSR & UDFS & RSFS & MCFS & LapScore \\
            \hline
             Lymphoma   &$\mathbf{0.5906}$  &$0.5528$   &$0.4830$       &$0.3778$   &$0.5003$  &$0.4670$  &$0.5524$  &$0.4701$ \\
             Lung  &$\mathbf{0.7255}$  &$0.6300$   &$0.5885$  &$0.6501$   &$0.6245$ &$0.5365$  &$0.6876$  &$0.5220$            \\
             WapPIE10P  &$\mathbf{0.4349}$   &$0.2570$  &$0.2648$ &$0.2597$ &$0.3809$  &$0.2794$  &$0.2884$   &$0.3327$            \\
             JAFFE   &$\mathbf{0.8532}$  &$0.7772$   &$0.7449$  &$0.7556$   &$0.7380$ &$0.6775$  &$0.7884$  &$0.6651$           \\
             Umist  &$\mathbf{0.5008}$  &$0.4598$   &$0.4235$  &$0.4715$   &$0.4924$ &$0.4402$  &$0.4708$  &$0.3998$            \\
             COIL20  &$\mathbf{0.6596}$   &$0.6416$   &$0.5670$   &$0.6440$  &$0.6024$  &$0.5674$  &$0.5637$  &$0.5700$     \\
             Isolet   &$\mathbf{0.5849}$  &$0.5836$   &$0.5153$    &$0.3979$   &$0.4837$  &$0.4969$  &$0.5571$  &$0.4643$          \\
             MFEA   &$\mathbf{0.6860}$  &$0.6220$   &$0.6290$  &$0.6284$  &$0.6260$   &$0.6267$  &$0.5685$    &$0.6592$          \\
             USPS   &$\mathbf{0.6667}$  &$0.6306$   &$0.5842$  &$0.5745$   &$0.6122$  &$0.5749$  &$0.6398$  &$0.3197$  \\

            \hline
            \multicolumn{9}{c}{Part B: Standard deviations}\\
            Data sets & AGUFS  &  URAFS& NDFS  & JELSR & UDFS & RSFS & MCFS & LapScore\\
            \hline
             Lymphoma   &$0.0275$  &$0.0301$   &$0.0386$       &$0.0418$   &$0.0291$  &$0.0720$  &$0.0718$  &$0.0354$          \\
             Lung  &$0.0734$  &$0.0923$   &$0.0599$  &$0.0550$   &$0.0698$ &$0.0676$  &$0.0853$  &$0.0472$            \\
             WapPIE10P  &$0.0311$   &$0.0236$  &$0.0163$ &$0.0132$ &$0.0175$  &$0.0240$  &$0.0323$   &$0.0306$            \\
             JAFFE   &$0.0617$  &$0.0934$   &$0.0399$  &$0.0457$   &$0.0339$ &$0.0500$  &$0.0735$  &$0.0314$           \\
             Umist  &$0.0258$  &$0.0249$   &$0.0231$  &$0.0309$   &$0.0198$ &$0.0264$  &$0.0243$  &$0.0162$            \\
             COIL20  &$0.0268$   &$0.0221$   &$0.0398$   &$0.0224$  &$0.0186$  &$0.0493$  &$0.0429$  &$0.0106$     \\
             Isolet   &$0.0239$  &$0.0409$   &$0.0405$    &$0.0185$   &$0.0125$  &$0.0265$  &$0.0366$  &$0.0152$          \\
             MFEA   &$0.0148$  &$0.0290$   &$0.0135$  &$0.0179$  &$0.0240$   &$0.0534$  &$0.0431$    &$0.0076$          \\
             USPS   &$0.0198$  &$0.0339$   &$0.0006$  &$0.0046$   &$0.0038$  &$0.0268$  &$0.0343$  &$0.0014$          \\
            \hline
        \end{tabular*}
    \end{flushleft}
\end{table*}

\begin{table*}[h]
\tabcolsep 0pt
\caption{Means and standard deviation of NMI for different unsupervised feature selection methods}\small \label{KNN}
\vspace*{-5pt}
    \begin{flushleft}
    \def\temptablewidth{\textwidth}
        {\rule{\temptablewidth}{1pt}}
        \begin{tabular*}{\temptablewidth}{@{\extracolsep{\fill}}lllllllll}
            %\hline
            \multicolumn{9}{c}{Part A: Means}\\
            Data sets  & AGUFS  &  URAFS& NDFS  & JELSR & UDFS & RSFS & MCFS & LapScore \\
            \hline
             Lymphoma   &$\mathbf{0.6846}$  &$0.6502$   &$0.5324$  &$0.3801$   &$0.6188$  &$0.4798$  &$0.5865$  &$0.5286$          \\
             Lung  &$\mathbf{0.6073}$  &$0.5034$   &$0.4951$  &$0.4970$   &$0.4961$ &$0.3644$  &$0.5206$  &$0.3844$            \\
             WapPIE10P  &$\mathbf{0.4482}$   &$0.2573$  &$0.2902$ &$0.2704$ &$0.3594$ &$0.2497$  &$0.2868$  &$0.3157$            \\
             JAFFE   &$\mathbf{0.8732}$  &$0.8182$   &$0.7989$  &$0.7715$   &$0.7154$ &$0.7296$  &$0.8457$  &$0.6961$           \\
             Umist  &$\mathbf{0.7098}$  &$0.6681$   &$0.6014$  &$0.6654$   &$0.6579$ &$0.6198$  &$0.6677$  &$0.5900$            \\
             COIL20  &$\mathbf{0.7601}$   &$0.7441$   &$0.6936$   &$0.7528$  &$0.6989$  &$0.7082$  &$0.7130$  &$0.6926$     \\
             Isolet   &$\mathbf{0.7010}$  &$0.6881$   &$0.6406$    &$0.5307$   &$0.6173$  &$0.6241$  &$0.6834$  &$0.6271$          \\
             MFEA   &$\mathbf{0.6679}$  &$0.6133$   &$0.6106$  &$0.6108$   &$0.6414$  &$0.6083$  &$0.5796$  &$0.6502$          \\
             USPS   &$\mathbf{0.6212}$  &$0.5749$   &$0.5557$  &$0.5160$  &$0.5525$  &$0.5338$  &$0.6114$  &$0.3313$          \\

            \hline
            \multicolumn{9}{c}{Part B: Standard deviations}\\
            Data sets & AGUFS  &  URAFS& NDFS  & JELSR & UDFS & RSFS & MCFS & LapScore\\
            \hline
             Lymphoma   &$0.0168$  &$0.0264$   &$0.0303$  &$0.0409$   &$0.0254$  &$0.0439$  &$0.0559$  &$0.0308$          \\
             Lung  &$0.0478$  &$0.0628$   &$0.0455$  &$0.0346$   &$0.0227$ &$0.0446$  &$0.0457$  &$0.0149$            \\
             WapPIE10P  &$0.0391$   &$0.0420$  &$0.0206$ &$0.0177$ &$0.0187$ &$0.0323$  &$0.0378$  &$0.0248$            \\
             JAFFE   &$0.0394$  &$0.0677$   &$0.0297$  &$0.0308$   &$0.0241$ &$0.0274$  &$0.0441$  &$0.0199$           \\
             Umist  &$0.0198$  &$0.0185$   &$0.0166$  &$0.0158$   &$0.0146$ &$0.0197$  &$0.0226$  &$0.0099$            \\
             COIL20  &$0.0175$   &$0.0129$   &$0.0200$   &$0.0119$  &$0.0078$  &$0.0256$  &$0.0176$  &$0.0078$     \\
             Isolet   &$0.0116$  &$0.0315$   &$0.0350$    &$0.0108$   &$0.0100$  &$0.0153$  &$0.0166$  &$0.0056$          \\
             MFEA   &$0.0078$  &$0.0165$   &$0.0075$  &$0.0085$   &$0.0131$  &$0.0284$  &$0.0227$  &$0.0175$          \\
             USPS   &$0.0164$  &$0.0251$   &$0.0011$  &$0.0055$  &$0.0029$  &$0.0084$  &$0.0103$  &$0.0005$          \\
            \hline
        \end{tabular*}
    \end{flushleft}
\end{table*}

\subsection{Parameter sensitivity}
In this subsection, we explore the parameter sensitivity of the proposed method AGUFS. Since the parameter $\beta$ is determined by the Eq. (\ref{22}), we focus on the impact parameters of $\alpha$ and $\lambda$. The clustering metrics ACC and NMI are utilized to evaluate the performance with the parameters  $\alpha$ and $\lambda$ searched in the grid $\left\{10^{-3}, 10^{-2}, 10^{-1}, 1, 10,10^{2}, 10^{3}\right\}$ and the number of features varied in $\left\{20,60,100,140,180\right\}$. Due to the limited space, we give the experimental results on three datasets Umist, COIL20 and JAFFE, as shown in Figs. 3 and 4 in terms of ACC and NMI, respectively. There are the similar results on other data sets. From Figs. 3 and 4, we can observe that ACC and NMI fluctuate a little with the varying $\alpha$ and $\lambda$. Namely, the proposed method AGUFS is not sensitive with regards to the parameters  $\alpha$ and $\lambda$ over a wide range. Then we can easily choose the parameters when using this method in practice.

\subsection{Convergence study}
The proposed Algorithm AGUFS for solving the objective function is iterative. In Section 5.1, we have already proven its convergence. In this subsection, we experimentally investigate the convergence speed of the proposed algorithm.

Fig. 5 shows the convergence curves of the proposed algorithm AGUFS on three data sets Umist, COIL20 and JAFFE. And the experimental results are similar on other data sets. In Fig. 5, the $x$-axis pertains to the iteration number and the $y$-axis pertains to the value of the objective function. As we can see, the convergence speed of the proposed algorithm is very fast, usually within 10 iterations. It further demonstrates the efficiency of the proposed algorithm AGUFS.

\begin{figure}[!htbp]
\centering
\includegraphics[width=\textwidth]{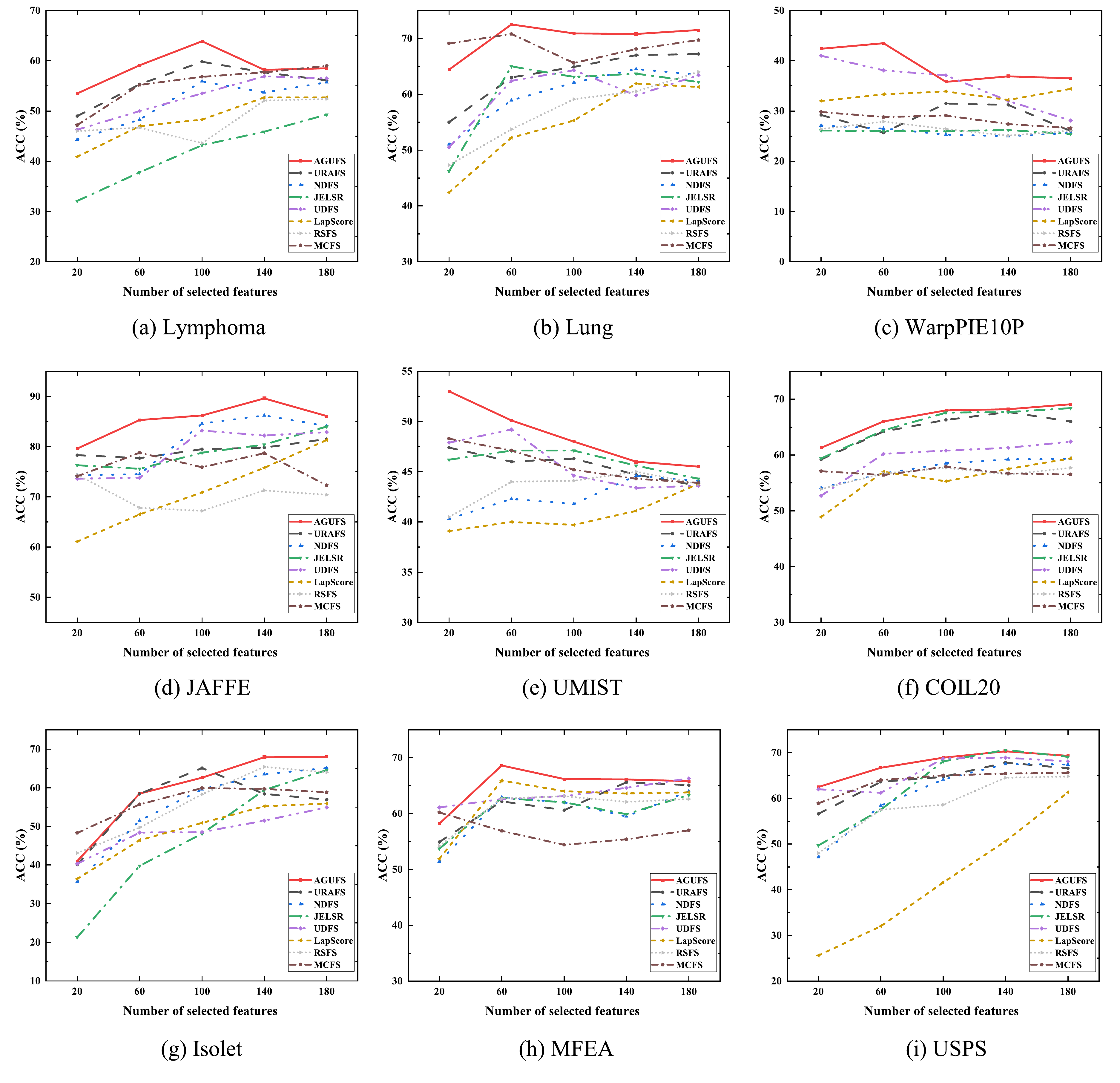}
\caption{ACC results compared with competing methods under different number of selected features}\small \label{FigAddStaDyn}
\end{figure}

\begin{figure}[!htbp]
\centering
\includegraphics[width=\textwidth]{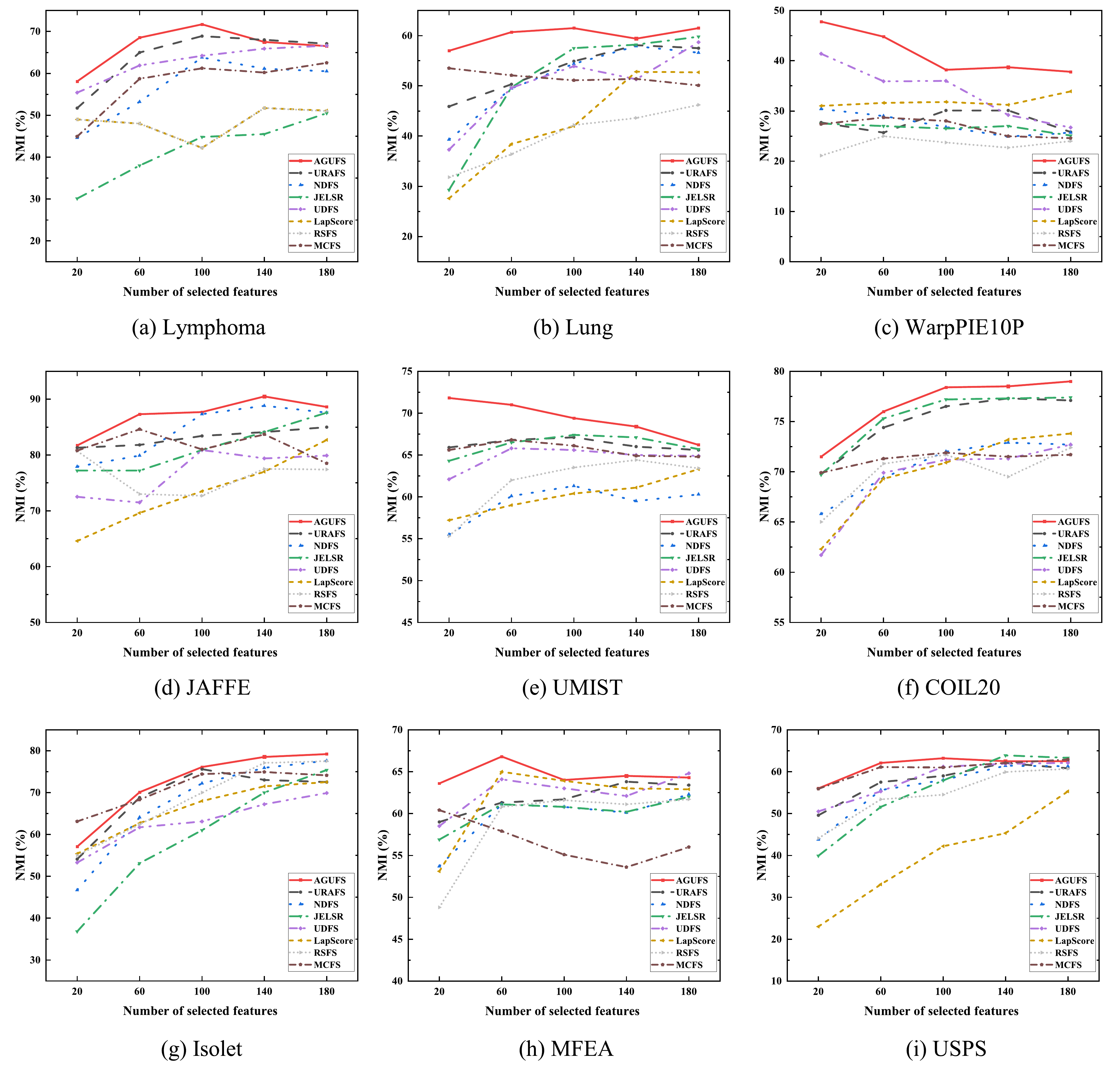}
\caption{NMI results compared with competing methods under different number of selected features}\small \label{FigAddStaDyn}
\end{figure}

\begin{figure}[!htbp]
\centering
\includegraphics[scale=0.44]{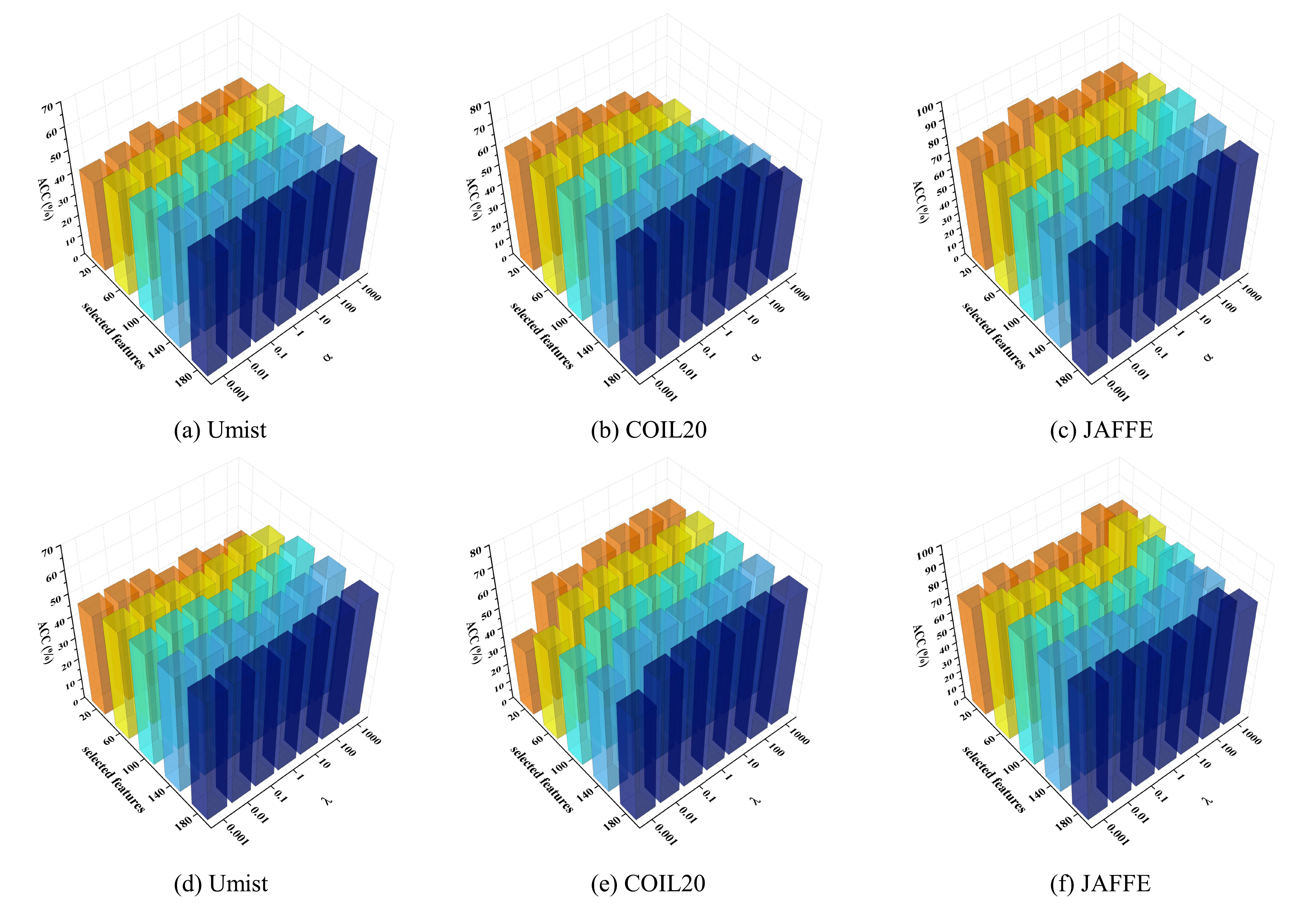}
\caption{ACC with varying  parameters $\alpha$ and $\lambda$}\small \label{FigAddStaDyn}
\end{figure}

\begin{figure}[!htbp]
\centering
\includegraphics[scale=0.44]{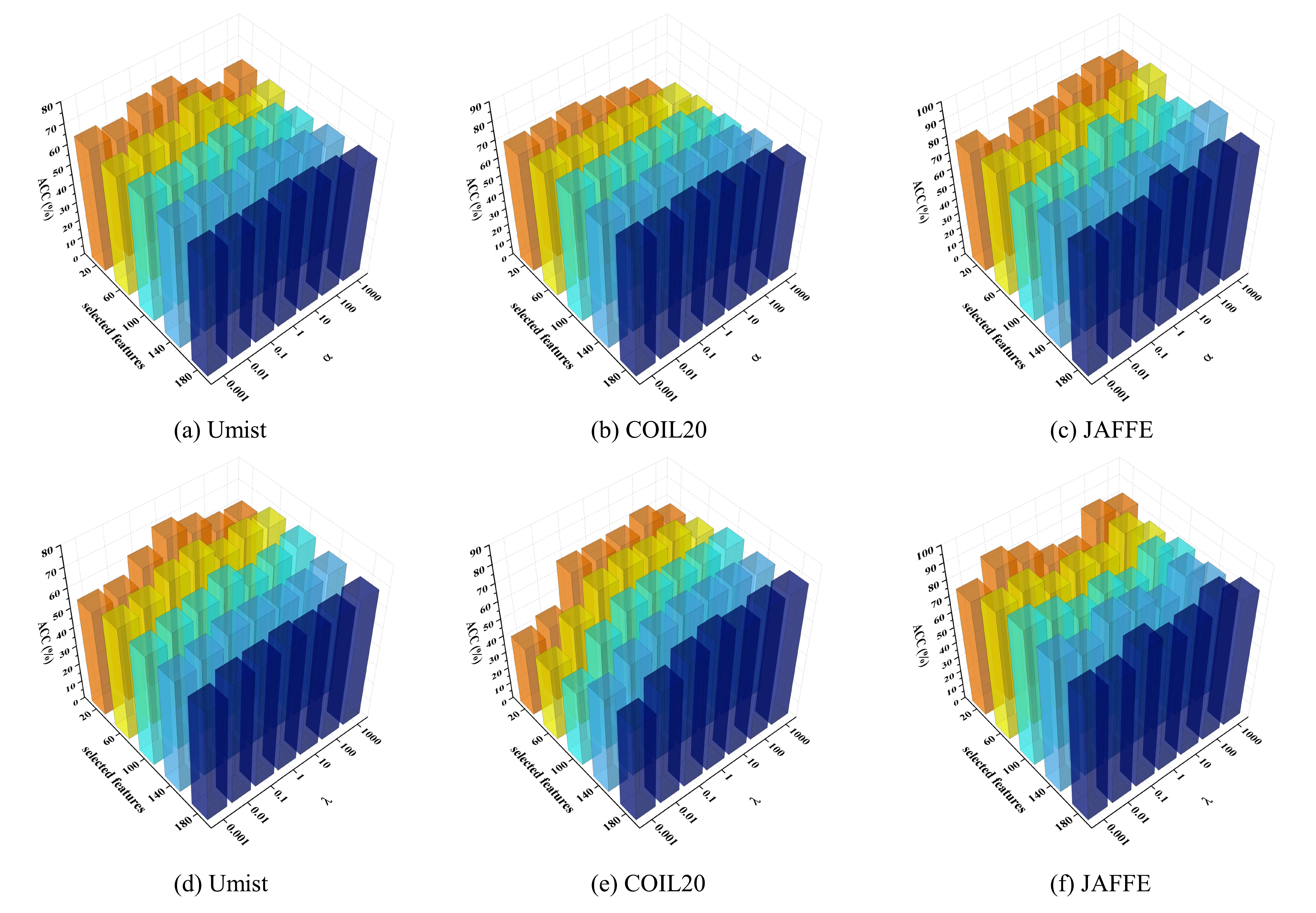}
\caption{NMI with varying  parameters $\alpha$ and $\lambda$}\small \label{FigAddStaDyn}
\end{figure}

\begin{figure}[!htbp]
\centering
\includegraphics[scale=0.44]{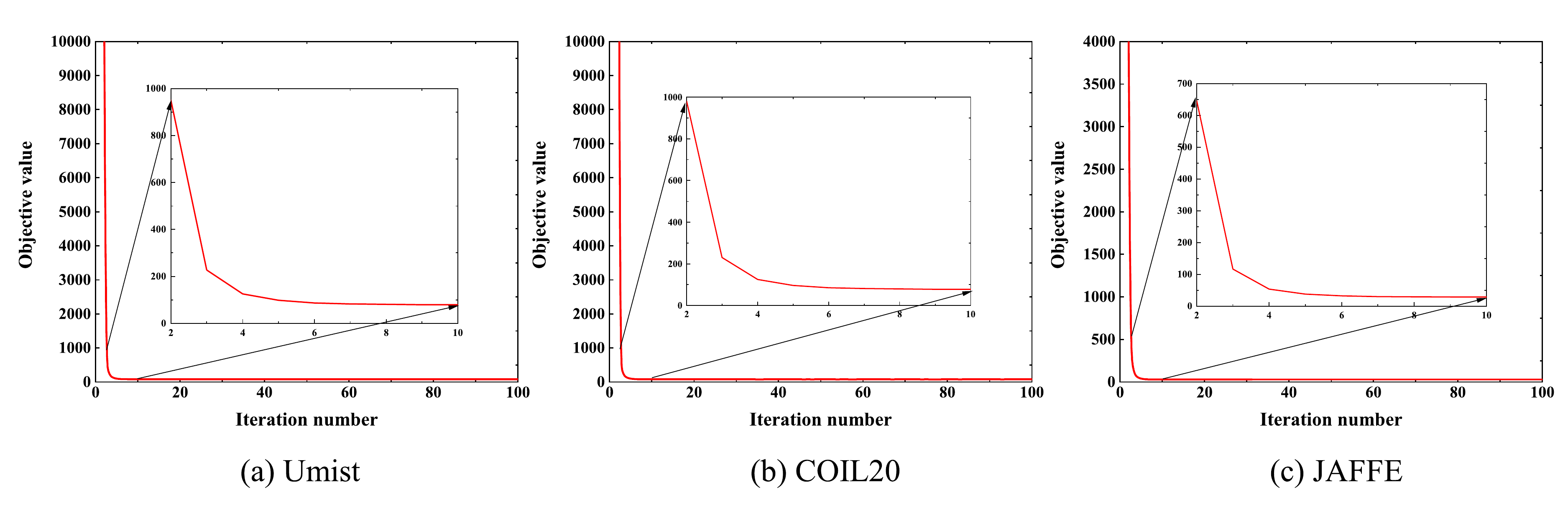}
\caption{Convergence curves of AGUFS on Umist, COIL20 and JAFFE.}\small \label{FigAddStaDyn}
\end{figure}

\section{Conclusion}
In this paper, we presented a generalized regression model with adaptive graph learning for unsupervised feature selection. To select the efficient features, the proposed method introduced a generalized regression model with an uncorrelated constraint, which can choose the discriminative and  uncorrelated features as well as reduce the variance of these data points belonging to the same neighborhood under the graph structure. Meanwhile, the adaptive learning of the similarity-induced graph and the spectral graph method-based learning of indicator matrix in local manifold are integrated into a coherent model  for unsupervised feature selection. An alternative iterative optimization algorithm was presented to solve the objective function.  The corresponding time complexity and convergence of the proposed algorithm were discussed.  Extensive experimental results on nine real-world data sets demonstrated the effectiveness and superiority of the proposed AGUFS method in comparison with other seven baseline approaches. In the future work, we will extend the proposed method to deal with the semi-supervised feature selection task, where integrating  the available label information into the low dimensional embedding structure. Moreover, we will further investigate how to speed up our method while coping with the large-scale data.

\section{Acknowledgements}
This work is supported by the National Science Foundation of China (Nos. 61602327, 61603313), the
Fundamental Research Funds for the Central Universities (No. 220710004005040177) and the Joint
Lab of Data Science and Business Intelligence at Southwestern University of Finance and Economics.

\end{document}